\documentclass[letterpaper]{article}

\usepackage[final,nonatbib]{nips_2018}

\usepackage[T1]{fontenc}    %
\usepackage{hyperref}       %
\usepackage{booktabs}       %
\usepackage{amsfonts}       %
\usepackage{nicefrac}       %
\usepackage{microtype}      %

\usepackage{todonotes}
\usepackage[numbers]{natbib}
\usepackage{floatrow}
\usepackage{graphicx}
\usepackage{subcaption}

\usepackage[inline]{enumitem}
\usepackage{color}
\usepackage[perpage]{footmisc}

\usepackage{amsmath}
\usepackage{amsfonts}
\usepackage{amssymb}
\usepackage{amsthm}
\usepackage{bm}
\usepackage{bbm}
\usepackage{mathtools}
\usepackage{enumitem}
\usepackage{thmtools,thm-restate}
\usepackage{algorithm}
\usepackage{algorithmic}
\usepackage{natbib}
\usepackage{color}
\usepackage{graphicx}
\usepackage{comment}
\usepackage{xspace}

\newcommand{\Coupled}{\textsc{Coupled}\xspace}
\newcommand{\CoupledNat}{\textsc{CoupledNat}\xspace}
\newcommand{\OrthDecoupled}{\textsc{Orth}\xspace}
\newcommand{\OrthDecoupledNat}{\textsc{OrthNat}\xspace}
\newcommand{\Hybrid}{\textsc{Hybrid}\xspace}
\newcommand{\Decoupled}{\textsc{Decoupled}\xspace}

\newcommand{\suppmat}{Appendix\xspace}  %
\newcommand{\dataset}[1]{\texttt{#1}}  %

\theoremstyle{plain}
\newtheorem{lemma}{Lemma}[section]

\newtheorem{proposition}{Proposition}[section]

\theoremstyle{definition}

\theoremstyle{remark}

\def\DD{\mathcal{D}}\def\EE{\mathcal{E}}
\def\GG{\mathcal{G}}\def\HH{\mathcal{H}}
\def\LL{\mathcal{L}}
\def\NN{\mathcal{N}}
\def\PP{\mathcal{P}}

\def\XX{\mathcal{X}}

\def\Ab{\mathbf{A}}\def\Bb{\mathbf{B}}
\def\Db{\mathbf{D}}
\def\Ib{\mathbf{I}}
\def\Kb{\mathbf{K}}\def\Lb{\mathbf{L}}

\def\Sb{\mathbf{S}}

\def\ab{\mathbf{a}}
\def\fb{\mathbf{f}}

\def\jb{\mathbf{j}}\def\kb{\mathbf{k}}
\def\mb{\mathbf{m}}

\def\yb{\mathbf{y}}

\def\Cbb{\mathbb{C}}
\def\Ebb{\mathbb{E}}

\def\Rbb{\mathbb{R}}

\def\R{\Rbb}\def\C{\Cbb}

\def\GP{\GG\PP}

\def\der{\mathrm{d}}

\def\diag{\mathrm{diag}}

\def\t{\top}

\newcommand{\abs}[1]{ \left| #1 \right|  }
\newcommand{\tr}[1]{ \mathrm{tr}\left( #1\right)}
\newcommand{\lr}[2]{ \left\langle #1, #2 \right\rangle}
\DeclareMathOperator*{\argmin}{arg\,min}

\newcommand{\KL}[2]{D_{KL}(#1, #2)} %

\newcommand{\E}{\Ebb}

\usepackage{etex,etoolbox}
\usepackage{environ}

\makeatletter
\providecommand{\@fourthoffour}[4]{#4}
\newcommand\fixstatement[2][\proofname\space of]{%
	\ifcsname thmt@original@#2\endcsname
	\AtEndEnvironment{#2}{%
		\xdef\pat@label{\expandafter\expandafter\expandafter
			\@fourthoffour\csname thmt@original@#2\endcsname\space\@currentlabel}%
		\xdef\pat@proofof{\@nameuse{pat@proofof@#2}}%
	}%
	\else
	\AtEndEnvironment{#2}{%
		\xdef\pat@label{\expandafter\expandafter\expandafter
			\@fourthoffour\csname #1\endcsname\space\@currentlabel}%
		\xdef\pat@proofof{\@nameuse{pat@proofof@#2}}%
	}%
	\fi
	\@namedef{pat@proofof@#2}{#1}%
}

\newcounter{proofcount}

\NewEnviron{proofatend}{%
	\edef\next{%
		\noexpand\begin{proof}[\pat@proofof\space\pat@label]%
			\unexpanded\expandafter{\BODY}}%
		\global\toks\numexpr\prooftoks+\value{proofcount}\relax=\expandafter{\next\end{proof}}
	\stepcounter{proofcount}}

\def\printproofs{%
	\count@=\z@
	\loop
	\the\toks\numexpr\prooftoks+\count@\relax
	\ifnum\count@<\value{proofcount}%
	\advance\count@\@ne
	\repeat}
\makeatother

\fixstatement{lemma}
\fixstatement{theorem}
\fixstatement{proposition}
\fixstatement{corollary}

\newcommand{\mbf}[1]{\mathbf{#1}}
\renewcommand{\vec}[1]{\mathrm{vec}(#1)}

\renewcommand{\KL}[2]{\mathrm{KL}(#1||#2)}

\title{Orthogonally Decoupled Variational\\ Gaussian Processes}

\graphicspath{ {figs/} }

\author{
  Hugh Salimbeni\footnote[1]{Equal contribution.}\\
  Imperial College London\\
  \texttt{hrs13@ic.ac.uk}
  \And
  Ching-An Cheng\footnotemark[1]\\
  Georgia Institute of Technology\\
  \texttt{cacheng@gatech.edu}
    \AND
  Byron Boots\\
  Georgia Institute of Technology\\
  \texttt{bboots@gatech.edu}
    \And
  Marc Deisenroth\\
  Imperial College London\\
  \texttt{mpd37@ic.ac.uk}
}

\begin{document}
\maketitle

\footnotetext[1]{Equal contribution}

\begin{abstract}
Gaussian processes (GPs) provide a powerful non-parametric framework for reasoning over functions. Despite appealing theory, its superlinear computational and memory complexities have presented a long-standing challenge. State-of-the-art sparse variational inference methods trade modeling accuracy against complexity. However, the complexities of these methods still  scale superlinearly in the number of basis functions, implying that that sparse GP methods are able to learn from large datasets only when a small model is used. Recently, a decoupled approach was proposed that removes the unnecessary coupling between the complexities of modeling the mean and the covariance functions of a GP. It achieves a linear complexity in the number of mean parameters, so an expressive posterior mean function can be modeled. While promising, this approach suffers from optimization difficulties due to ill-conditioning and non-convexity. In this work, we propose an alternative decoupled parametrization. It adopts an orthogonal basis in the mean function to model the residues that cannot be learned by the standard coupled approach. Therefore, our method extends, rather than replaces, the coupled approach to achieve strictly better performance. This construction admits a straightforward natural gradient update rule, so the structure of the information manifold that is lost during decoupling can be leveraged to speed up learning. Empirically, our algorithm demonstrates significantly faster convergence in multiple experiments.
\end{abstract}

\section{Introduction}
Gaussian processes (GPs) are flexible Bayesian non-parametric models that have achieved state-of-the-art performance in a range of applications~\citep{diggle2007springer,Snoek2012}. A key advantage of GP models is that they have large representational capacity, while being robust to overfitting~\citep{Rasmussen2006}. %
This property is especially important for robotic applications, where there may be an abundance of data in some parts of the space but a scarcity in others~\citep{deisenroth2011pilco}. Unfortunately, exact inference in GPs scales cubically in computation and quadratically in memory with the size of the training set, and is only available in closed form for Gaussian likelihoods.

To learn from large datasets, variational inference provides a principled way to find tractable approximations to the true posterior. A common approach to approximate GP inference is to form a sparse variational posterior, which is designed by conditioning the prior process at a small set of \emph{inducing points}~\citep{titsias2009variational}. The sparse variational framework  trades accuracy against computation, but its complexities still scale superlinearly in the number of inducing points. %
 Consequently, the representation power of the approximate distribution is greatly limited.

Various attempts have been made to reduce the complexities in order to scale up GP models for better approximation.
 Most of them, however, rely on certain assumptions on the kernel structure and input dimension. In the extreme, \citet{hartikainen2010kalman} show that, for 1D-input problems, exact GP inference can be solved in linear time for kernels with finitely many non-zero derivatives. 
For low-dimensional inputs and stationary kernels, variational inference with 
structured kernel approximation~\citep{wilson2015kernel} or Fourier features~\citep{hensman2016variational} has been proposed.
Both approaches, nevertheless, scale exponentially with input dimension,
except for the special case of sum-and-product kernels~\citep{gardner2018product}. 
Approximate kernels have also been proposed as GP priors with low-rank structure~\citep{snelson2006sparse, quinonero2005unifying} or a sparse spectrum~\citep{quia2010sparse}. 
Another family of methods partitions the input space into subsets and performs prediction aggregation~\citep{tresp2000bayesian, rasmussen2002infinite, nguyen2014fast,deisenroth2015distributed}, and
Bayesian aggregation of local experts with attractive theoretical properties is recently proposed by~\citet{rulliere2018nested}.

A recent decoupled framework~\citep{cheng2017variational} takes a different direction to address the complexity issue of GP inference. 
In contrast to the above approaches, this decoupled framework is agnostic to problem setups (e.g. likelihoods, kernels, and input dimensions) and extends the original sparse variational formulation~\citep{titsias2009variational}. 
The key idea is to represent the variational distribution in the reproducing kernel Hilbert space (RKHS) induced by the covariance function of the GP. The sparse variational posterior by~\citet{titsias2009variational} turns out to be equivalent to a particular parameterization in the RKHS, where the mean and covariance both share the same basis. 
\citet{cheng2017variational} suggest to relax the requirement of basis sharing. Since the computation only scales linearly in the mean parameters, many more basis functions can be used for modeling the mean function to achieve higher accuracy in prediction. 

However, the original decoupled basis~\citep{cheng2017variational} turns out to have optimization difficulties~\citep{havasi2018deep}. 
In particular, the non-convexity of the optimization problem means that a suboptimal solution may be found, leading to performance that is potentially worse than the standard coupled case. 
While \citet{havasi2018deep} suggest to use a pre-conditioner to amortize the problem, their algorithm incurs an additional cubic computational cost; therefore, its applicability is limited to small simple models.%

Inspired by the success of natural gradients in variational inference~\citep{hoffman2013stochastic, salimbeni2018natural}, we propose a novel RKHS parameterization of decoupled GPs that admits efficient natural gradient computation. We decompose the mean parametrization into a part that shares the basis with the covariance, and an orthogonal part that models the residues that the standard coupled approach fails to capture. We show that, with this particular choice, the natural gradient update rules further \emph{decouple} into the natural gradient descent of the coupled part and the functional gradient descent of the residual part. Based on these insights, we propose an efficient optimization algorithm that preserves the desired properties of decoupled GPs and  converges faster than the original formulation~\citep{cheng2017variational}.

We demonstrate that our basis is more effective than the original decoupled formulation on a range of classification and regression tasks. We show that the natural gradient updates improve convergence considerably and can lead to much better performance in practice. Crucially, we show also that our basis is more effective than the standard coupled basis for a fixed computational budget.

\section{Background}

We consider the inference problem of GP models. Given a dataset $\DD = \{(x_n, y_n) \}_{n=1}^N$ and a GP prior on a latent function $f$, the goal is to infer the (approximate) posterior of $f(x^*)$ for any query input $x^*$. In this work, we adopt the recent decoupled RKHS reformulation of variational inference~\citep{cheng2017variational}, and,  without loss of generality, we will assume $f$ is a scalar function.
For notation, we use boldface to distinguish finite-dimensional vectors and matrices that are used in computation from scalar and abstract mathematical objects.

\subsection{Gaussian Processes and their RKHS Representation} \label{sec:GP and RKHS}

We first review the primal and dual representations of GPs, which form the foundation of the RKHS reformulation.  
Let $\XX \subseteq \R^d$ be the domain of the latent function. A GP is a distribution of functions, which is described by a mean function $m: \XX \to \R$ and a covariance function $k: \XX \times \XX \to \R$. 
We say a latent function $f$ is distributed according to  $\GP(m, k)$, if for any $x, x' \in \XX$, 
$\E[f(x)] = m(x)$,  $\C[f(x), f(x')] = k(x,x')$, and  for any finite subset $\{ f(x_n) : x_n \in \XX \}_{n=1}^N$ is Gaussian distributed. %

We call the above definition, in terms of the \emph{function values} $m(x)$ and $k(x,x')$, the \emph{primal} representation of GPs. Alternatively, one can  adopt a \emph{dual} representation of GPs, by treating \emph{functions} $m$ and $k$ as RKHS objects~\citep{cheng2016incremental}. 
This is based on observing that the covariance function $k$ satisfies the definition of positive semi-definite functions, %
so $k$ can also be viewed as a reproducing kernel \citep{aronszajn1950theory}.  
Specifically, given $\GP(m, k)$, without loss of generality, we can find an RKHS $\HH$ such that 
\begin{align}\label{eq:RKHS form}
	m(x) = \phi(x)^\t \mu, \qquad k(x,x') = \phi(x)^\t \Sigma \phi(x')
\end{align}
for some  $\mu  \in \HH$, bounded positive semidefinite self-adjoint operator $\Sigma: \HH \to \HH$, and feature map $\phi: \XX \to \HH$. Here we use $^\t$ to denote the inner product in $\HH$, even when $\dim{\HH}$ is infinite. 
For notational clarity, we use symbols $m$ and $k$ (or $s$) to denote the mean and covariance functions, and symbols $\mu$ and $\Sigma$ to denote the RKHS objects; we use $s$ to distinguish the (approximate) posterior covariance function from the prior covariance function $k$. 
If $ f \sim \GP(m, k)$ satisfying~\eqref{eq:RKHS form}, we also write $f \sim \GP_\HH( \mu, \Sigma)$.\footnote{ 
	This notation only denotes that $m$ and $k$ can be represented as RKHS objects, not that the sampled functions of $\GP(m, k)$ necessarily reside in $\HH$ (which only holds for the special when $\Sigma$ has finite trace).} 

To concretely illustrate the primal-dual connection, we consider the GP regression problem. Suppose $f \sim \GP(0,k)$ in prior and $y_n = f(x_n) + \epsilon_n$, where $\epsilon_n \sim \NN(\epsilon_n | 0, \sigma^2)$.
Let $X = \{x_n\}_{n=1}^N$ and $\yb = (y_n)_{n=1}^N \in \R^N$, where the notation $(\cdot)_{n = \cdot}^\cdot$ denotes stacking the elements. 
Then, with $\DD$  observed, it can be shown that $f \sim \GP(m, s)$ where 
\begin{align} \label{eq:GPR}
m(x) = \kb_{x,X} (\Kb_X + \sigma^2 \Ib)^{-1} \yb, \qquad
s(x,x') = k_{x,x'} -  \kb_{x,X} (\Kb_X + \sigma^2 \Ib)^{-1} \kb_{X,x'},
\end{align}
where $k_{\cdot,\cdot}$, $\mbf{k_{\cdot,\cdot}}$ and $\mbf{K}_{\cdot,\cdot}$ denote the covariances between the subscripted sets,\footnote{If the two sets are the same, only one is listed.}
We can also equivalently write the posterior GP in~\eqref{eq:GPR} in its dual RKHS representation: suppose the feature map $\phi$ is selected such that $k(x,x') = \phi(x)^\t \phi(x')$, then \textit{a priori} $f \sim  \GP_\HH( 0, I)$ and \textit{a posteriori} %
$f \sim \GP_\HH( \mu, \Sigma)$, %
\begin{align} \label{eq:GPR in RKHS}
\mu = \Phi_{X} (\Kb_X + \sigma^2 \Ib)^{-1} \yb, \qquad
\Sigma  = I - \Phi_{X} (\Kb_X + \sigma^2 \Ib)^{-1} \Phi_{X}^\t,
\end{align}
where $\Phi_{X} = [\phi(x_1), \dots, \phi(x_N)]$.

\subsection{Variational Inference Problem}

Inference in GP models is challenging because the closed-form expressions in~\eqref{eq:GPR} have computational complexity that is cubic in the size of the training dataset, and are only applicable for Gaussian likelihoods.
For non-Gaussian likelihoods (e.g. classification) or for large datasets (i.e. more than 10,000 data points), we must adopt approximate inference.

Variational inference provides a principled approach to search for an approximate but tractable posterior. 
It seeks a variational posterior $q$  that is close to the true posterior $p(f | \DD)$ in terms of KL divergence, i.e. it solves $\min_{q}  \KL{q(f)}{p(f | \DD)}$. 
For GP models, the variational posterior must be defined over the entire function, so a natural choice is to use another GP. This choice is also motivated by the fact that the exact posterior is a GP in the case of a Gaussian likelihood as shown in~\eqref{eq:GPR}. 
Using the results from Section~\ref{sec:GP and RKHS}, we can represent this posterior process via a mean and a covariance function or, equivalently, through their associated RKHS objects. 

We denote these RKHS objects as $\mu$ and $\Sigma$, which uniquely determine the GP posterior $ \GP_\HH( \mu, \Sigma)$. In the following, without loss of generality, we shall assume  that the prior GP is zero-mean and the RKHS is selected such that $f \sim \GP_\HH( 0, I)$ \textit{a priori}.

The variational inference problem in GP models leads to the optimization problem 
\begin{align} \label{eq:VI problem}
\min_{q = \GP_\HH( \mu, \Sigma) } \LL (q)\,,
\qquad 
\LL(q) =  - \sum\nolimits_{n=1}^{N} \E_{q(f(x_n))}[ \log p(y_n| f(x_n))] + \KL{q(f)}{p(f)}\,,
\end{align}
where $p(f) = \GP_\HH( 0, I)$ and  $\KL{q(f)}{p(f)}= \int  \log \frac{q(f)}{p(f)} \der q(f)$ is the KL divergence between the approximate posterior GP $q(f)$ and the prior GP $p(f)$. It can be shown that $\LL(q) = \KL{q(f)}{p(f|\DD)}$ up to an additive constant~\citep{cheng2017variational}.

\subsection{Decoupled Gaussian Processes}

Directly optimizing the possibly infinite-dimensional RKHS objects $\mu$ and $\Sigma$ is computationally intractable except for the special case of a Gaussian likelihood and a small training set size $N$. Therefore, in practice, we need to impose a certain sparse structure on $\mu$ and $\Sigma$.
Inspired by the functional form of the exact solution in~\eqref{eq:GPR in RKHS}, \citet{cheng2017variational} propose to model the approximate posterior GP in the \emph{decoupled subspace parametrization} (which we will refer to as \emph{decoupled basis} for short) with
\begin{align} \label{eq:general subspace parametrization}
\mu  = \Psi_\alpha \ab, \qquad 
\Sigma = I + \Psi_\beta \Ab \Psi_\beta^\t
\end{align}
where $\alpha$ and $\beta$ are the sets of \emph{inducing points} to specify the bases $\Psi_\alpha$ and $\Psi_\beta$ in the RKHS, and $\ab \in \R^{|\alpha|}$ and $\Ab \in \R^{|\beta| \times |\beta|}$ are the coefficients such that $\Sigma \succeq 0$. With only finite perturbations from the prior, the construction in~\eqref{eq:general subspace parametrization} ensures the KL divergence $\KL{q(f)}{p(f)}$ is finite~\citep{matthews2016sparse,cheng2017variational} (see Appendix~\ref{app:variational inference}). 
Importantly, this parameterization decouples the variational parameters $(\ab, \alpha)$ for the mean $\mu$ and the variational parameters $(\Ab, \beta)$ for the covariance $\Sigma$. As a result, the computation complexities related to the two parts become independent, and a \emph{large} set of parameters can adopted for the mean to model complicated functions, as discussed below.

\paragraph{Coupled Basis}
The form in~\eqref{eq:general subspace parametrization} covers the  sparse variational posterior~\citep{titsias2009variational}. 
Let $Z=\{z_n \in\XX \}_{n=1}^{M}$ be some fictitious inputs and let 
$\fb_{Z} = (f(z_n))_{n=1}^M$ be the vector of function values. 
Based on the primal viewpoint of GPs, \citet{titsias2009variational} constructs the variational posterior as the posterior GP conditioned on $Z$ with  marginal   $q(\fb_{Z})=\mathcal N(\fb_{Z} | \mb, \Sb)$, 
where $\mb \in \R^M$ and $\Sb \succeq 0 \in \R^{M \times M}$. 
The elements in $Z$  along with $\mb$ and $\Sb$ are the variational parameters to optimize.
The mean and covariance functions of this process $\GP(m,s)$ are %
\begin{align}
m(x) =  \kb_{x, Z} \Kb^{-1}_{Z} \mb, \qquad 
s(x,x') = k_{x, x'} + \kb_{x, Z} \mbf{K}_{Z}^{-1} (  \Sb - \mbf{K}_{Z}) \mbf{K}_{Z}^{-1}  \kb_{Z, x} \,,
\end{align}
which is reminiscent of the exact result in~\eqref{eq:GPR}.
Equivalently, it has the dual representation 
\begin{align} \label{eq:coupled RKHS}
\mu =   \Psi_{Z} \Kb^{-1}_{Z}\mb,
\qquad 
\Sigma = I + \Psi_{Z} \mbf{K}_{Z}^{-1} ( \Sb - \mbf{K}_{Z}) \mbf{K}_{Z}^{-1}  \Psi_{Z}^\t ,
\end{align}
which conforms with the form in~\eqref{eq:general subspace parametrization}. The computational complexity of using the coupled basis reduces from $O(N^3)$ to $O(M^3 + M^2 N )$. Therefore, when $M \ll N$ is selected, the GP can be applied to learning from large datasets~\citep{titsias2009variational}.

\paragraph{Inversely Parametrized Decoupled Basis}%
Directly parameterizing the dual representation in~\eqref{eq:general subspace parametrization} admits more flexibility than the primal function-valued perspective.
To ensure that the covariance of the posterior strictly decreases compared with the prior, \citet{cheng2017variational} propose a decoupled basis with an inversely parametrized covariance operator
\begin{align} \label{eq:decoupled RKHS (cheng 2017)}
\mu = \Psi_\alpha \ab, \qquad 
\Sigma = (I + \Psi_\beta \Bb \Psi_\beta^\t)^{-1},
\end{align}
where $\Bb \succeq 0 \in \R^{|\beta| \times |\beta|}$ and is further parametrized by its Cholesky factors in implementation. It can be shown that the choice in~\eqref{eq:decoupled RKHS (cheng 2017)} is equivalent to setting $\Ab = - \Kb_\beta^{-1} +  (\Kb_\beta  \Bb \Kb_\beta + \Kb_\beta )^{-1}$ 
in~\eqref{eq:general subspace parametrization}.
In this parameterization, because the bases for the mean and the covariance are decoupled, the computational complexity of solving~\eqref{eq:VI problem} with the decoupled basis in~\eqref{eq:decoupled RKHS (cheng 2017)}
becomes $O(|\alpha| + |\beta|^3)$, as opposed to $O(M^3)$ of~\eqref{eq:coupled RKHS}. 
Therefore, while it is usually assumed that $|\beta| $ is in the order of $M$, with a decoupled basis, we can freely choose  $|\alpha| \gg |\beta|$ for modeling complex mean functions accurately.

\section{Orthogonally Decoupled Variational Gaussian Processes }

While the particular decoupled basis in~\eqref{eq:decoupled RKHS (cheng 2017)} is more expressive, its optimization problem is ill-conditioned and non-convex, and empirically slow convergence has been observed~\citep{havasi2018deep}. 
To improve the speed of learning decoupled models, %
we consider the use of natural gradient descent~\citep{amari1998natural}. In particular, we are interested in the update rule for \emph{natural parameters}, which has empirically demonstrated impressive convergence performance over other choices of parametrizations~\citep{salimbeni2018natural}

However, it is unclear what the natural parameters~\eqref{eq:general subspace parametrization} for the general decoupled basis in~\eqref{eq:general subspace parametrization} are and whether finite-dimensional natural parameters even exist for such a model. In this paper, we show that when a decoupled basis is appropriately structured, then natural parameters do exist. Moreover, they admit a very efficient (approximate) natural gradient update rule as detailed in Section~\ref{sec:NGD}.
As a result, large-scale decoupled models can be quickly learned, joining the fast convergence property from the coupled approach~\citep{hensman2013gaussian} and the flexibility of the decoupled approach~\citep{cheng2017variational}.

\subsection{Alternative Decoupled Bases}
To motivate the proposed approach, let us first introduce some alternative decoupled bases for improving optimization properties~\eqref{eq:decoupled RKHS (cheng 2017)} and discuss their limitations. The inversely parameterized decoupled basis~\eqref{eq:decoupled RKHS (cheng 2017)} is likely to have different optimization properties from the standard coupled basis~\eqref{eq:coupled RKHS}, due to the inversion in its covariance parameterization. To avoid these potential difficulties, we reparameterize the covariance of~\eqref{eq:decoupled RKHS (cheng 2017)} as the one in~\eqref{eq:coupled RKHS} and consider instead the basis
\begin{align} \label{eq:decoupled_aS}
\mu = \Psi_\alpha  \ab, 
\qquad
\Sigma = (I - \Psi_\beta \Kb^{-1}_\beta \Psi_\beta^\top) + 
\Psi_\beta \Kb^{-1}_{\beta} \Sb\Kb^{-1}_{\beta}\Psi_\beta ^\top\,.
\end{align}
The basis \eqref{eq:decoupled_aS} can be viewed as a decoupled version of \eqref{eq:coupled RKHS}: it can be readily identified that setting $\alpha=\beta=Z$ and  $\ab=\Kb_Z^{-1}\mb$ recovers \eqref{eq:coupled RKHS}. Note that we do not want to define a basis in terms of $\Kb_\alpha^{-1}$ as that incurs the cubic complexity that we intend to avoid.  This basis gives a posterior process with
\begin{align} \label{eq:decoupled_1_posterior}
m(x) = \kb_{x,\alpha} \ab, \qquad
s(x, x') = k_{x,x'} - \kb_{x,\beta} \Kb^{-1}_\beta (\Sb -  \Kb_\beta)\Kb^{-1}_\beta\kb_{\beta, x'}.
\end{align}

The alternate choice~\eqref{eq:decoupled_aS} addresses the difficulty in optimizing the covariance operator, but it still suffers from one serious drawback: while using more inducing points, \eqref{eq:decoupled_aS} is not necessarily more expressive than the standard basis~\eqref{eq:coupled RKHS}, for example, when $\alpha$ is selected badly. To eliminate the worst-case setup, we can explicitly consider $\beta$ to be part of $\alpha$ and use
\begin{align}\label{eq:decoupled_gamma_beta_S}
\mu = \Psi_\gamma \ab_{\gamma} + \Psi_\beta \Kb_\beta^{-1}\mb_{\beta}, 
\qquad
\Sigma = (I - \Psi_\beta \Kb^{-1}_\beta \Psi_\beta^\top) + 
\Psi_\beta \Kb^{-1}_{\beta} \Sb\Kb^{-1}_{\beta}\Psi_\beta ^\top.
\end{align}
where $\gamma = \alpha \setminus \beta$.
This is exactly the \emph{hybrid basis} suggested in the appendix of \citet{cheng2017variational}, which is strictly more expressive than~\eqref{eq:coupled RKHS} and yet has the complexity as~\eqref{eq:decoupled RKHS (cheng 2017)}. Also the explicit inclusion of $\beta$ inside $\alpha$ is pivotal to defining proper finite-dimensional natural parameters, which we will later discuss.   This basis gives a posterior process with the same covariance as \eqref{eq:decoupled_1_posterior}, and mean $m(x) = \kb_{x,\gamma} \ab_\gamma  +  \kb_{x,\beta}\Kb_{\beta}^{-1}\mb_\beta $.

\subsection{Orthogonally Decoupled Representation} \label{sec:orthogonal}
But is~\eqref{eq:decoupled_gamma_beta_S} the best possible decoupled basis? Upon closer inspection, we find that there is redundancy in the parameterization of this basis: as $\Psi_\gamma$ is not orthogonal to $\Psi_\beta$ in general, optimizing $\ab_{\gamma}$ and $\mb_{\beta}$ jointly would create coupling and make the optimization landscape more ill-conditioned.

To address this issue, under the partition that $\alpha = \{ \beta, \gamma\}$, we propose a new decoupled basis as
\begin{align} \label{eq:orthogonal RKHS}
\mu = (I - \Psi_\beta \Kb^{-1}_\beta \Psi_\beta^\top)\Psi_\gamma \ab_{\gamma} + \Psi_\beta \ab_{\beta}, 
\qquad
\Sigma = (I - \Psi_\beta \Kb^{-1}_\beta \Psi_\beta^\top) + 
\Psi_\beta \Kb^{-1}_{\beta} \Sb\Kb^{-1}_{\beta}\Psi_\beta ^\top,
\end{align}
where $\ab_\gamma \in \R^{|\gamma|}$, $\ab_{\beta} \in \R^{|\beta|}$ and $\Sb =\Lb\Lb^{\top}$ is  parametrized by its Cholesky factor.
We call $(\ab_\gamma, \ab_\beta, \Sb)$ the \emph{model parameters} 
and 
refer to \eqref{eq:orthogonal RKHS} as the \emph{orthogonally} decoupled basis, because $(I - \Psi_\beta \Kb^{-1}_\beta \Psi_\beta^\top)$ is orthogonal to $\Psi_\beta$ (i.e. $(I - \Psi_\beta \Kb^{-1}_\beta \Psi_\beta^\top)^\top\Psi_\beta=0$). 
By substituting  $Z = \beta$ and $\ab_\beta = \Kb_Z^{-1} \mb$, we can compare~\eqref{eq:orthogonal RKHS} to~\eqref{eq:coupled RKHS}: \eqref{eq:orthogonal RKHS} has an additional part parameterized by $\ab_\gamma$ to model the mean function residues that \emph{cannot} be captured by using the inducing points $\beta$ alone. In prediction, our basis has time complexity in $O(|\gamma| + |\beta|^3)$ because $\Kb_{\beta}^{-1} \Kb_{\beta, \gamma}  \ab_{\gamma}$ can be precomputed.  The orthogonally decoupled basis results in a posterior process with
\begin{align*}%
m(x) = (\kb_{x,\gamma} - \kb_{x, \beta}  \Kb_{\beta, \gamma} ) \ab_{\gamma} + \kb_{x, \beta} \ab_{\beta}, \quad 
s(x, x') = k_{x,x'} - \kb_{x,\beta} \Kb^{-1}_{\beta} (\Sb-\Kb^{-1}_{\beta})\Kb^{-1}_{\beta} \kb_{\beta, x'}.
\end{align*}
This decoupled basis can also be derived from the perspective of \citet{titsias2009variational} by conditioning the prior on a finite set of inducing points\footnote{We thank an anonymous reviewer for highlighting this connection.}. Details of this construction are in Appendix \ref{app:inducing_point_construction}.

Compared with the original decoupled basis in~\eqref{eq:decoupled RKHS (cheng 2017)}, our choice in~\eqref{eq:orthogonal RKHS} has attractive properties:
\begin{enumerate}%
\item The explicit inclusion of $\beta$ as a subset of $\alpha$ leads to the existence of natural parameters. 
\item If the likelihood is strictly log-concave (e.g. Gaussian and Bernoulli likelihoods), then the variational inference problem in~\eqref{eq:VI problem} is strictly convex  in $(\ab_\gamma, \ab_\beta, \Lb)$ (see Appendix~\ref{app:convexity}).
\end{enumerate}

Our setup in~\eqref{eq:orthogonal RKHS} introduces a projection operator before $\Psi_{\gamma}\ab_\gamma$ in the basis~\eqref{eq:decoupled_gamma_beta_S} and therefore it can be viewed as the \emph{unique} hybrid parametrization, which confines the function modeled by $\gamma$ to be orthogonal to the span the $\beta$ basis. Consequently, there is no correlation between optimizing $\ab_\gamma$ and $\ab_\beta$, making the problem more well-conditioned.

\subsection{Natural Parameters and Expectation Parameters} 
\label{sec:parameters}

To identify the natural parameter of GPs structured as~\eqref{eq:orthogonal RKHS},
we revisit the definition of natural parameters in exponential families. A distribution $p(x)$ belongs to an exponential family if we can write $p(x) = h(x) \exp( t(x)^\t\eta - A(\eta))$, where $t(x)$ is the sufficient statistics, $\eta$ is the natural parameter, $A$ is the log-partition function, and $h(x)$ is the carrier measure. 

Based on this definition, we can see that the choice of natural parameters is not unique. Suppose $\eta = H \tilde{\eta} + b$ for some constant matrix $H$ and vector $b$. Then $\tilde{\eta}$ is also an admissible natural parameter, because we can write $p(x) = \tilde{h}(x) \exp( \tilde{t}(x)^\t \tilde{\eta} - \tilde{A}(\tilde\eta))$, 
where $\tilde{t}(x) = H^\t t(x)$, $\tilde{h}(x) = h(x) \exp(t(x)^\t b)$, and $\tilde{A}(\tilde{\eta}) = A(H\tilde{\eta} + b)$. 
In other words, the natural parameter is only unique up to affine transformations.
If the natural parameter is transformed, the corresponding expectation parameter $\theta = \E_{p}[t(x)]$ also transforms accordingly to $\tilde{\theta} = H^\t \theta$. It can be shown that the Legendre primal-dual relationship between $\eta$ and $\theta$ is also preserved:  $\tilde{A}$ is also convex, and it satisfies $\tilde{\theta} = \nabla \tilde{A}(\tilde{\eta})$ and $\tilde{\eta} = \nabla \tilde{A}^*(\tilde{\theta})$, where $*$ denotes the Legendre dual function (see Appendix~\ref{app:nat param}).

We use this trick to identify the natural and expectation parameters of~\eqref{eq:orthogonal RKHS}.\footnote{While GPs do not admit a density function, the property of transforming natural parameters described above still applies. An alternate proof can be derived using KL divergence.
}
The relationships between natural, expectation, and model parameters are summarized in Figure~\ref{fig:relationship}.
\paragraph{Natural Parameters}
Recall that for Gaussian distributions the natural parameters are conventionally defined as $(\Sigma^{-1} \mu, \frac{1}{2}\Sigma^{-1})$. Therefore, to find the natural parameters of~\eqref{eq:orthogonal RKHS}, it suffices to show that $(\Sigma^{-1} \mu, \frac{1}{2}\Sigma^{-1})$ of~\eqref{eq:orthogonal RKHS} can be written as an affine transformation of some finite-dimensional parameters.
The matrix inversion lemma and 
the orthogonality 
of $(I - \Psi_\beta\Kb_\beta^{-1}\Psi_\beta^{\top})$ and $\Psi_\beta$
yield
\begin{align} \label{eq:natural parameters}
&\Sigma^{-1} \mu = ( I - \Psi_{\beta} \Kb_{\beta}^{-1} \Psi_{\beta}^\t ) \Psi_{\gamma} \jb_\gamma 
+ \Psi_{\beta} \jb_\beta \,,\qquad \tfrac{1}{2}\Sigma^{-1} = \tfrac{1}{2}( I - \Psi_\beta\Kb_\beta^{-1} \Psi_\beta^{\top})  +  \Psi_\beta  \bm{\Theta} \Psi_\beta^{\top}, \nonumber \\
&\textstyle
\text{where} \qquad 
\jb_\gamma = \ab_\gamma,
\qquad 
\jb_\beta = \Sb^{-1}  \Kb_\beta \ab_\beta, 
\qquad 
\bm{\Theta} =  \frac{1}{2} \Sb^{-1}.
\end{align}

Therefore, we call $(\jb_\gamma, \jb_\beta, \bm\Theta)$ the natural parameters of~\eqref{eq:orthogonal RKHS}. This choice is unique in the sense that $\Sigma^{-1} \mu $ is orthogonally parametrized.\footnote{The hybrid parameterization~\eqref{eq:decoupled_gamma_beta_S} 	 in~\citep[Appendix]{cheng2017variational}, which also considers $\beta$ explicitly in $\mu$, admits natural parameters as well. However, their relationship and the natural gradient update rule turn out to be more convoluted; we provide a thorough discussion in Appendix~\ref{app:nat param}.} 
The explicit inclusion of $\beta$ as part of $\alpha$ is important; otherwise there will be a constraint on $\jb_\alpha$ and $\jb_\beta$ because $\mu$ can only be parametrized by the $\alpha$-basis (see Appendix~\ref{app:nat param}).

\paragraph{Expectation Parameters} Once the new natural parameters are selected, we can also derive the corresponding expectation parameters. Recall for the natural parameters  $(\Sigma^{-1} \mu, \frac{1}{2}\Sigma^{-1})$, the associated expectation parameters are $(\mu, -(\Sigma + \mu\mu^\t) )$. 
Using the relationship between transformed natural and expectation parameters,  we find the expected parameters of~\eqref{eq:orthogonal RKHS} using the adjoint operators:
$ 
[
( I - \Psi_{\beta} \Kb_{\beta}^{-1} \Psi_{\beta}^\t )  \Psi_{\gamma},
\Psi_{\beta} 
]^\t \mu
= 
[
\mb_{\gamma \perp \beta}, 
\mb_{\beta}
]^\t$
and
$
- \Psi_\beta^\t (\Sigma + \mu\mu^\t) \Psi_\beta = \bm\Lambda
$,
where we have
\begin{align} \label{eq:expectation parameters}
\mb_{\gamma \perp \beta} = (\Kb_{\gamma}   - \Kb_{\gamma, \beta}  \Kb_\beta^{-1} \Kb_{\beta, \gamma} ) \jb_\gamma, 
\qquad 
\mb_{ \beta} = \Sb \jb_\beta, 
\qquad 
\bm\Lambda = - \Sb -  \mb_{ \beta}\mb_{ \beta}^\t.
\end{align}
Note the equations for $\beta$ in~\eqref{eq:natural parameters} and~\eqref{eq:expectation parameters}  have exactly the same relationship between the natural and expectation parameters in the standard Gaussian case, i.e. $(\Sigma^{-1} \mu, \frac{1}{2}\Sigma^{-1}) \leftrightarrow (\mu, -(\Sigma + \mu\mu^\t) )$.

\begin{figure}[t]
	\includegraphics[width=\textwidth]{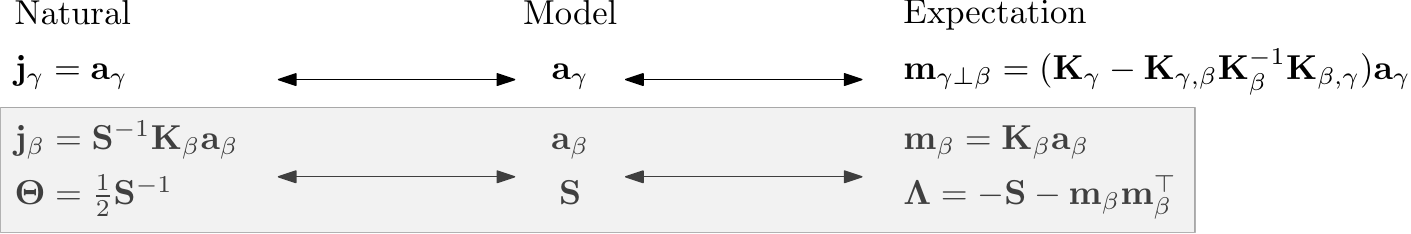}
	\caption{The relationship between the three parameterizations of the orthogonally decoupled basis. The box highlights the parameters in common with the standard coupled basis, which are decoupled from the additional $\ab_\gamma$ parameter. This is a unique property of our orthogonal basis}
   	\label{fig:relationship}
\end{figure}

\subsection{Natural Gradient Descent} \label{sec:NGD}
Natural gradient descent updates parameters according to the information geometry induced by the KL divergence~\citep{amari1998natural}. It is invariant to reparametrization and can normalize the problem to be well conditioned~\citep{martens2014new}. 
Let $F(\eta) = \nabla^2 \KL{q}{p_\eta}|_{q = p_\eta}$ be the Fisher information matrix, where $p_\eta$ denotes a distribution with natural parameter $\eta$. Natural gradient descent for natural parameters performs the update 
$\eta \leftarrow \eta -  \tau F(\eta)^{-1}\nabla_{\eta} \LL,$
where  $\tau > 0$ is the step size. Because directly computing the inverse $ F(\eta)^{-1}$  is computationally expensive, we use the duality between natural and expectation parameters in exponential families
and adopt the equivalent update 
$\eta \leftarrow \eta -  \tau  \nabla_{\theta} \LL$~\cite{hoffman2013stochastic,salimbeni2018natural}.
\paragraph{Exact Update Rules}
For our basis in~\eqref{eq:orthogonal RKHS}, the natural gradient descent step can be written  as
\begin{align}  \label{eq:ngd orthogonal GP}
\jb \leftarrow   \jb  - \tau \nabla_{\mb} \LL, 
\qquad
\bm\Theta \leftarrow \bm\Theta  - \tau \nabla_{\bm\Lambda} \LL,
\end{align}
where we recall $\LL$ is the negative variational lower bound in~\eqref{eq:VI problem}, $\jb = [\jb_\gamma, \jb_\beta]$ in~\eqref{eq:natural parameters}, and $\mb = [\mb_{\gamma \perp \beta}, \mb_\beta]$ in~\eqref{eq:expectation parameters}. 
As $\LL$ is defined in terms of $(\ab_{\gamma}, \ab_{\beta}, \Sb)$, to compute these derivatives we use chain rule (provided by the relationship in Figure~\ref{fig:relationship}) and obtain
\begin{subequations} \label{eq:ngd orthogonal GP (details)}
\begin{align}  
\jb_\gamma &\leftarrow   \jb_\gamma  - \tau  (\Kb_{\gamma} - \Kb_{\gamma, \beta} \Kb_\beta^{-1} \Kb_{\gamma, \beta} )^{-1}  \nabla_{\ab_\gamma} \LL,  \label{eq:ngd for j_gamma}
\\
\jb_\beta &\leftarrow   \jb_\beta  - \tau  ( \Kb_{\beta}^{-1} \nabla_{\ab_\beta} \LL  - 2  \nabla_{\Sb} \LL  \mb_{ \beta}  ), \label{eq:ngd for j_beta}
\\
\bm\Theta &\leftarrow \bm\Theta  + \tau \nabla_{\Sb} \LL.
\label{eq:ngd for Theta}
\end{align}
\end{subequations}
Due to the orthogonal choice of natural parameter definition, the update for the $\gamma$ and the $\beta$ parts are independent. Furthermore, one can show that the update for $\jb_\beta$ and $\bm\Theta$ is exactly the same as the natural gradient descent rule for the standard coupled basis~\citep{hensman2013gaussian}, and that the update for the residue part $\jb_\gamma$ is equivalent to  functional gradient descent~\citep{kivinen2004online} in the  subspace orthogonal to the span of $\Psi_\beta$.%

\paragraph{Approximate Update Rule}
\label{sec:preconditioning}
We described the natural gradient descent update for the orthogonally decoupled GPs in~\eqref{eq:orthogonal RKHS}. However, in the regime where $\abs{\gamma} \gg \abs{\beta}$, computing~\eqref{eq:ngd for j_gamma} becomes infeasible. Here we propose an approximation of~\eqref{eq:ngd for j_gamma} by approximating $\Kb_\gamma$ with a diagonal-plus-low-rank structure. Because the inducing points $\beta$ are selected to globally approximate the function landscape, one sensible choice is to approximate $\Kb_\gamma$ with a Nystr\"om approximation based on $\beta$ and a diagonal correction term: $\Kb_\gamma \approx \Db_{\gamma | \beta} +  \Kb_{\gamma | \beta}$, where
$\Db_{\gamma | \beta} = \diag(\Kb_\gamma - {\Kb}_{\gamma | \beta})$,  ${\Kb}_{\gamma | \beta} = \Kb_{\gamma, \beta} \Kb_\beta^{-1} \Kb_{\beta, \gamma} $, and $\diag$ denotes extracting the diagonal part of a matrix. FITC~\citep{snelson2006sparse} uses a similar idea to approximate the prior distribution~\citep{quinonero2005unifying}, whereas here it is used to derive an approximate update rule without changing the problem. 
This leads to a simple  update rule
\begin{align} \label{eq:approximate ngd for j_beta}
\jb_\gamma \leftarrow   \jb_\gamma  - \tau  (\Db_{\gamma | \beta} + \epsilon \Ib )^{-1}  \nabla_{\ab_\gamma} \LL,
\end{align}
where a jitter $\epsilon > 0$ is added to ensure stability. 
This update rule uses a diagonal scaling  $ (\Db_{\gamma | \beta} + \epsilon \Ib )^{-1}$, which is independent of $\jb_\beta$ and $\bm{\Theta}$. Therefore, while one could directly use the update~\eqref{eq:approximate ngd for j_beta}, in implementation, we propose to replace~\eqref{eq:approximate ngd for j_beta} with an adaptive coordinate-wise gradient descent algorithm (e.g. ADAM~\cite{Kingma2014Adam:Optimization}) to update the $\gamma$-part.  
Due to the orthogonal structure, the overall computational complexity is in $O(|\gamma||\beta| + |\beta|^3)$. While this is more than the $O(|\gamma| + |\beta|^3)$ complexity of the original decoupled approach~\citep{cheng2017variational}; the experimental results suggest the additional computation is worth the large performance improvement.

\section{Results}
We empirically assess the performance of our algorithm in multiple regression and classification tasks. We show that
\begin{enumerate*}[label=\itshape\alph*\upshape)]
\item  given fixed wall-clock time, the proposed orthogonally decoupled basis outperforms existing approaches;
\item given the same number of inducing points for the covariance, our method almost always improves on the coupled approach (which is in contrast to the previous decoupled bases);
\item using natural gradients can dramatically improve performance, especially in regression.
\end{enumerate*}

We compare updating our orthogonally decoupled basis with adaptive gradient descent using the Adam optimizer \citep{Kingma2014Adam:Optimization} (\OrthDecoupled), and using the approximate natural gradient descent rule described in Section \ref{sec:preconditioning} (\OrthDecoupledNat).
As baselines, we consider the original decoupled approach of \citet{cheng2017variational} (\Decoupled) %
and the hybrid approach suggested in their Appendix (\Hybrid).  We compare also to the standard coupled basis with and without natural gradients (\CoupledNat and \Coupled, respectively). We make generic choices for hyperparameters, inducing point initializations, and data processing, which are detailed in \suppmat~\ref{sec:experimental details}. 
Our code \footnote{\href{https://github.com/hughsalimbeni/orth\_decoupled\_var\_gps}{\color{blue}\texttt{https://github.com/hughsalimbeni/orth\_decoupled\_var\_gps}}} and datasets \footnote{\href{https://github.com/hughsalimbeni/bayesian_benchmarks}{\color{blue}\texttt{https://github.com/hughsalimbeni/bayesian\_benchmarks}}} are publicly available.

\begin{figure*}[t!]
    \centering
    \begin{subfigure}{0.33\textwidth}
		\centering
        \includegraphics[height=3.2cm]{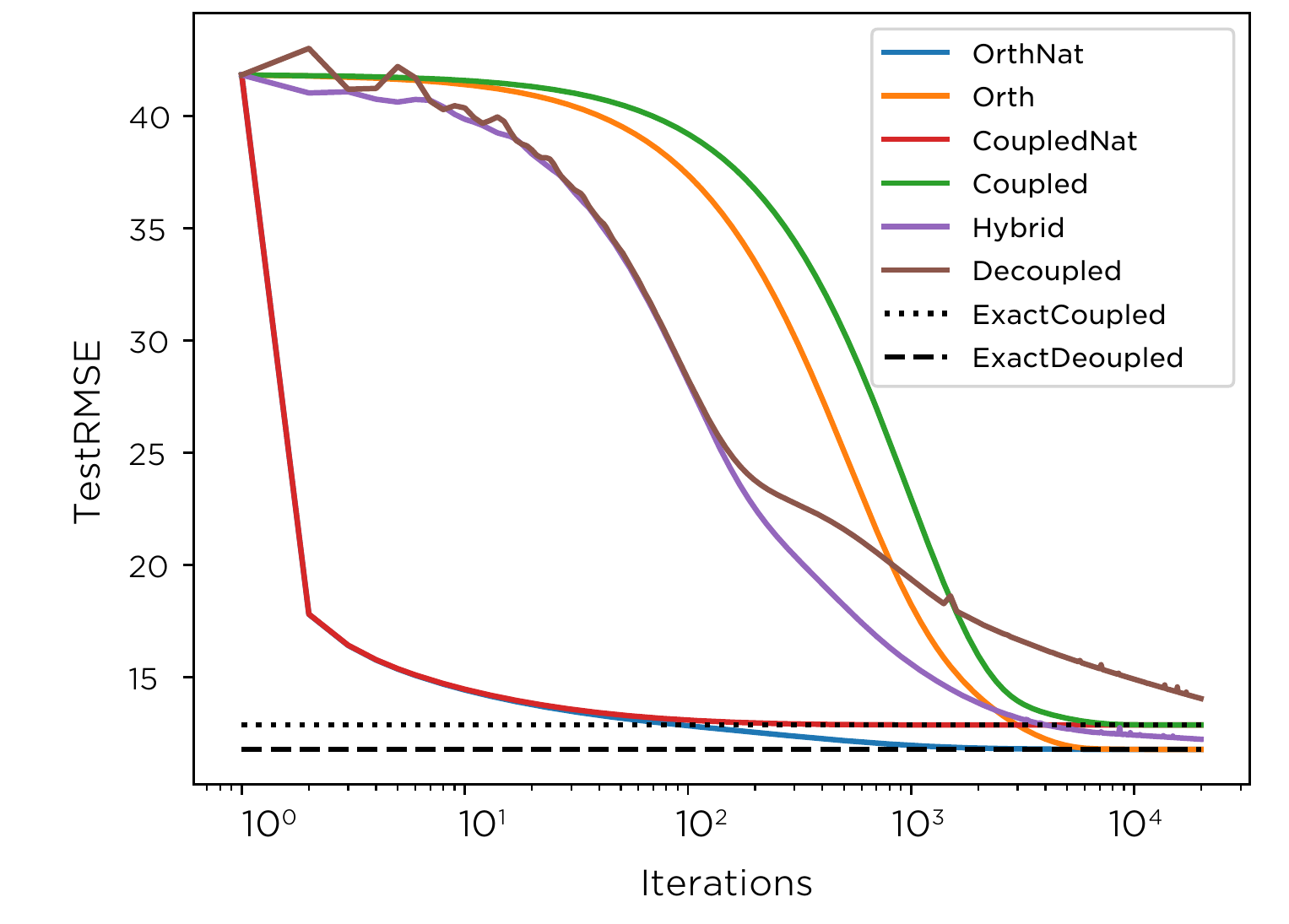}
        \caption{Test RMSE}
         \label{fig:test rmse}
    \end{subfigure}%
    \begin{subfigure}{0.33\textwidth}
        \centering
        \includegraphics[height=3.2cm]{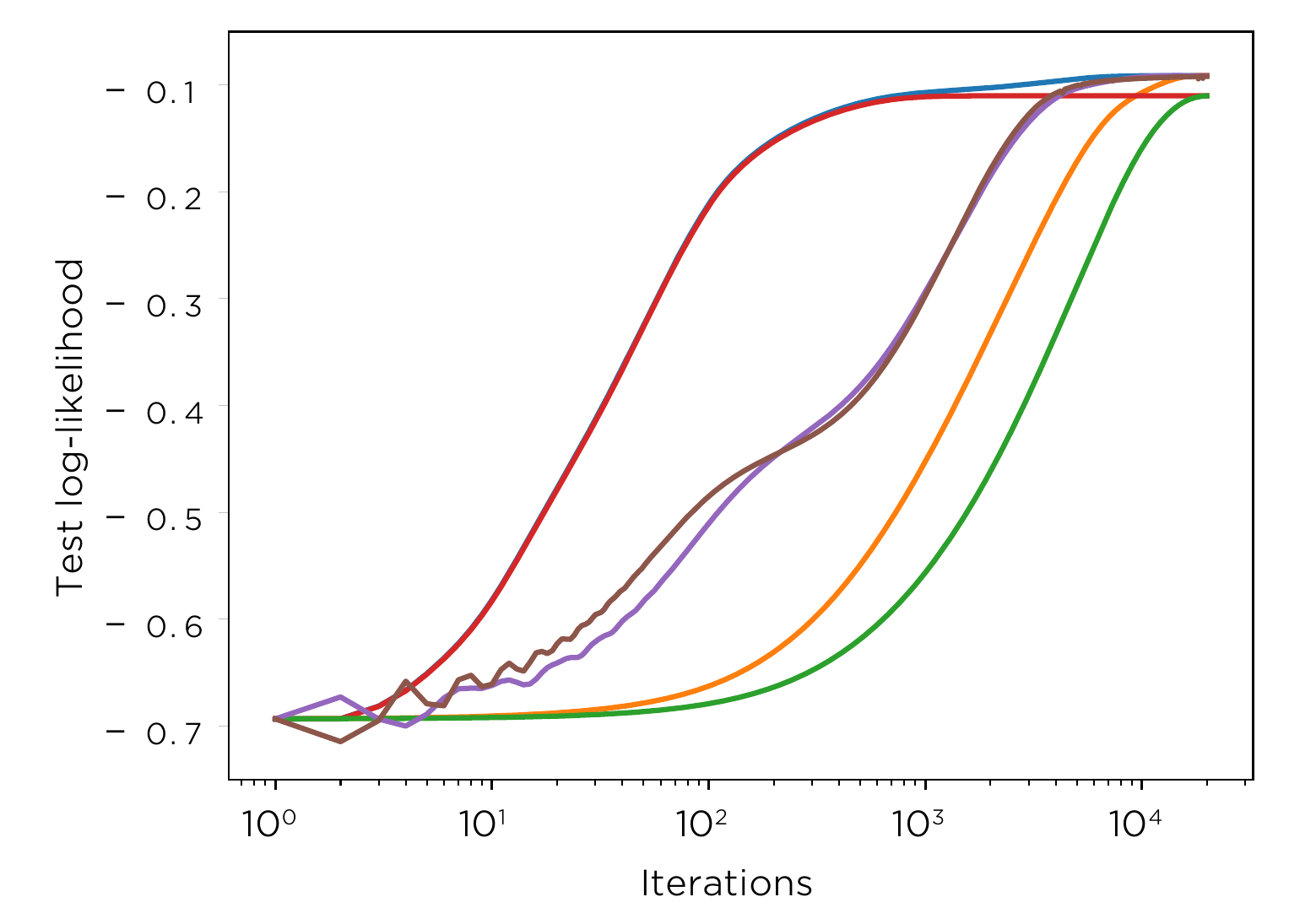}
        \caption{Test log-likelihood}
        \label{fig:test log-likelihood}
    \end{subfigure}
    \begin{subfigure}{0.33\textwidth}
        \centering
        \includegraphics[height=3.2cm]{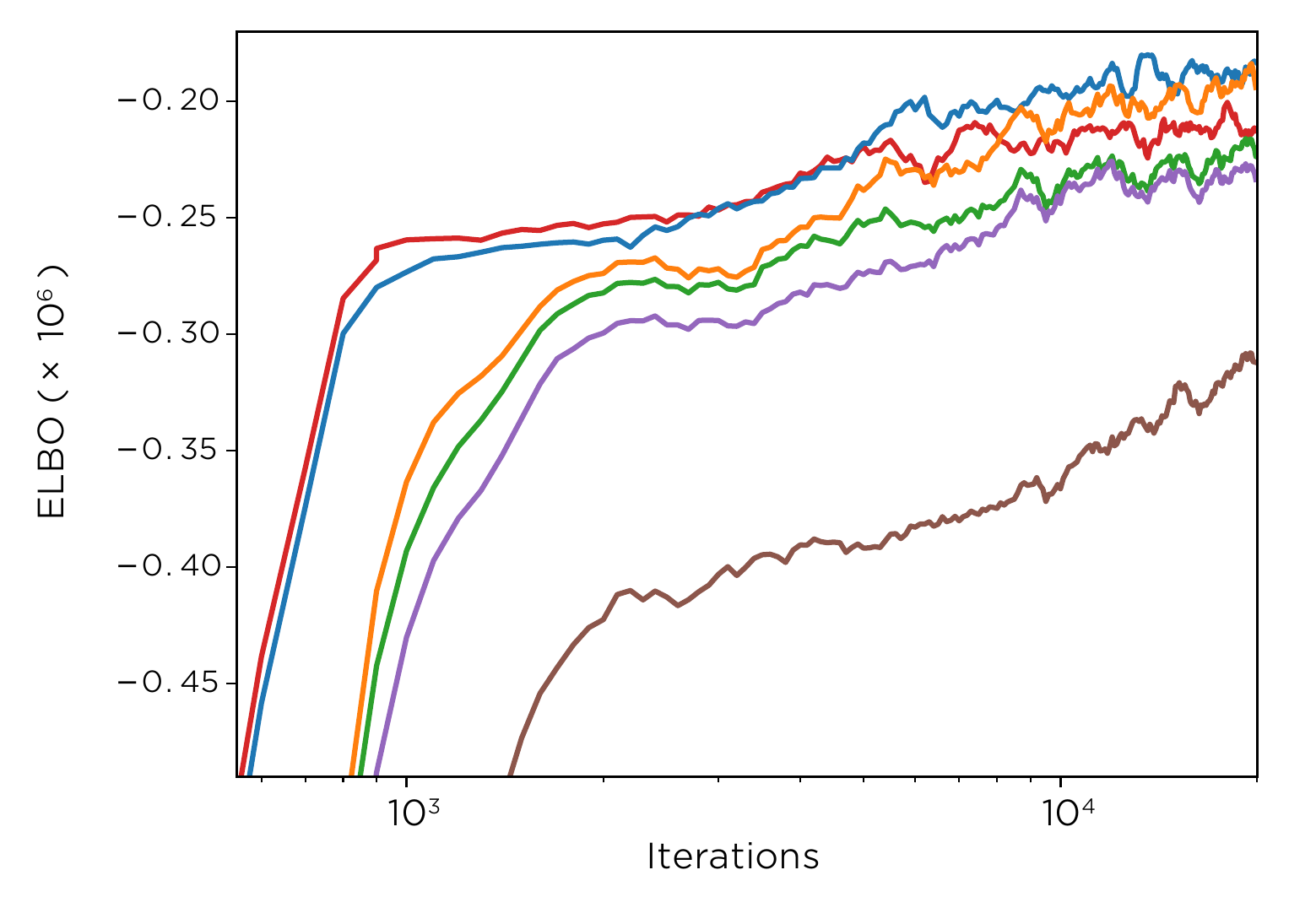}
        \caption{ELBO}
        \label{fig:elbo}
    \end{subfigure}
    \caption{Training curves for our models in three different settings. %
    Panel~(\subref{fig:elbo}) has $|\gamma|=3500, |\beta|=1500$ for the decoupled bases and $|\beta|=2000$ for the coupled bases. Panels~(\subref{fig:test rmse}) and~(\subref{fig:test log-likelihood}) have $|\gamma|=|\beta|=500$ and fixed hyperparameters and full batches to highlight the convergence properties of the approaches. Panels~(\subref{fig:test rmse}) and ~(\subref{fig:elbo}) use a Gaussian likelihood. Panel~(\subref{fig:test log-likelihood}) uses a Bernoulli likelihood.} 
\end{figure*}

\paragraph{Illustration} Figures \ref{fig:test rmse} and \ref{fig:test log-likelihood} show a simplified setting to illustrate the difference between the methods. In this example, we fixed the inducing inputs and hyperparameters, and optimized  the rest of the variational parameters. All the decoupled methods then have the same global optimum, so we can easily assess their convergence property. With a Gaussian likelihood we also computed the optimal solution analytically as an optimal baseline, %
although this requires inverting an $|\alpha|$-sized matrix, and therefore is not useful as a practical method. We include the expressions of the optimal solution in  \suppmat~\ref{sec:analytic_expressions}. We set $|\gamma|=|\beta|=500$ for all bases and conducted experiments on  \dataset{3droad} dataset ($N=434874$, $D=3$) for regression with a Gaussian likelihood and \dataset{ringnorm} data ($N=7400$, $D=21$) for classification with a Bernoulli likelihood. 
Overall, the natural gradient methods are much faster to converge than their ordinary gradient counterparts. \Decoupled  fails to converge to the optimum after 20K iterations, even in the Gaussian case. We emphasize that, unlike our proposed approaches, \Decoupled leads to a non-convex optimization problem.

\paragraph{Wall-clock comparison} To investigate large-scale performance, we used \dataset{3droad} with a large number of inducing points. We used a computer with a Tesla K40 GPU and found that, in wall-clock time, the orthogonally decoupled basis with $|\gamma|=3500,~|\beta|=1500$ was equivalent to a coupled model with $|\beta|=2000$ (about $0.7$ seconds per iteration) in our tensorflow~\citep{abadi2016tensorflow} implementation. %
Under this setting, we show the ELBO in Figure~\ref{fig:elbo} and the test log-likelihood and accuracy in Figure~\ref{fig:extra_plots} of the \suppmat~\ref{sec:further_results}. \OrthDecoupledNat performs the best, both in terms of log-likelihood and accuracy. The  the highest test log-likelihood of \OrthDecoupledNat  is $-3.25$, followed by  \OrthDecoupled  ($-3.26$).  \Coupled ($-3.37$) and \CoupledNat ($-3.33$)  both outperform \Hybrid ($-3.39$) and \Decoupled ($-3.66$).

\paragraph{Regression benchmarks} We applied our models on 12 regression datasets ranging from 15K to 2M points. To enable feasible computation on multiple datasets, we downscale (but keep the same ratio) the number of inducing points to $|\gamma|=700,|\beta|=300$ for the decoupled models, and $|\beta|=400$ for the coupled mode. 
We compare also to the coupled model with $|\beta|=300$ to establish whether extra computation always improves the performance of the decoupled basis. The  test mean absolute error (MAE) results are summarized in Table~\ref{table:regression_acc}, and the full results for both test log-likelihood and MAE are given in Appendix~\ref{sec:further_results}.  \textsc{OrthNat} overall is the most competitive basis. And, by all measures, the orthogonal bases outperform their coupled counterparts with the same $\beta$, except for \textsc{Hybrid} and \textsc{Decoupled}.

\begin{table}[h]
\scalebox{0.72}{
\begin{tabular}{lccccccccc} %
\toprule
                &       \textsc{Coupled}$\dagger$ & \textsc{CoupledNat}$\dagger$ & \textsc{Coupled} & \textsc{CoupledNat} & \textsc{Orth} &    \textsc{OrthNat} & \textsc{Hybrid} & \textsc{Decoupled} \\
\midrule
           Mean &                           0.298 &                       0.295 &           0.291 &              0.290 &  0.284 &  \textbf{0.282} &                   0.298 &             0.361 \\
         Median &                           0.221 &                        0.219 &            0.215 &               0.213 &  0.211 &   \textbf{0.210} &                     0.225 &              0.299 \\
      Avg Rank &  6.083(0.19)&	5.00(0.33)&	3.750(0.26)&	2.417(0.31)&	2.500(0.47)&	\textbf{1.833(0.35)}&	6.417(0.23)&	8(0.00)\\
\bottomrule
\end{tabular}
}
\caption{Regression summary for normalized test MAE on 12 regression datasets, with standard errors for the average ranks. The coupled bases had $|\beta|=400$ ($|\beta|=300$ for the $\dagger$ bases), and the decoupled all had $\gamma=700$, $\beta=300$. See Appendix~\ref{sec:further_results} for the full results.} 
\label{table:regression_acc}
\end{table}

\paragraph{Classification benchmarks}
We compare our method with state-of-the-art fully connected neural networks with Selu activations~\citep{klambauer2017self}. We adopted the experimental setup from \citep{klambauer2017self}, using the largest 19 datasets (4898 to 130000 data points). For the binary datasets we used the Bernoulli likelihood, and for the multiclass datasets we used the robust-max likelihood~\citep{hernandez2011robust}. The same basis settings as for the regression benchmarks were used here. 
\OrthDecoupled performs the best in terms of median, and  \OrthDecoupledNat is best ranked. The neural network wins in terms of mean, because it substantially outperforms all the GP models in one particular dataset (\dataset{chess-krvk}), which skews the mean performance over the 19 datasets.  We see that our orthogonal bases on average improve the coupled bases with equivalent wall-clock time, although for some datasets the coupled bases are superior. Unlike in the regression case, it is not always true that using natural gradients improve performance, although on average they do. This holds for both the coupled and decoupled bases. 

\begin{table}[h]
\label{table:classification}
\centering
\scalebox{0.85}{
\begin{tabular}{lccccccc}%
\toprule
                    &             Selu & \textsc{Coupled} & \textsc{CoupledNat} &    \textsc{Orth} & \textsc{OrthNat} &  \textsc{Hybrid} & \textsc{Decoupled} \\
\midrule              Mean &                    \textbf{91.6} &             90.4 &                90.2 &             90.6 &             90.3 &             89.9 &               89.0 \\
             Median &                    93.1 &             94.8 &                93.6 &             \textbf{95.6} &             93.6 &             93.4 &               92.0 \\
           Average rank &4.16(0.67)&	3.89(0.42)&	3.53(0.45)	&3.68(0.35)&	\textbf{3.42(0.31)}	&3.89(0.38)&	5.42(0.51) \\
           
\bottomrule
\end{tabular}
}
\caption{Classification test accuracy(\%) results for our models, showing also the results from \cite{klambauer2017self}, with standard errors for the average ranks. See Table~\ref{table:classification_acc_full} in the \suppmat for the complete results.}
\label{table:classification_acc}
\end{table}

Overall, the empirical results demonstrate  that the orthogonally decoupled basis is superior to the coupled basis with the same wall-clock time, averaged over datasets. It is important to note that for the \emph{same} $\beta$, adding extra $\gamma$ increases performance for the orthogonally decoupled basis in almost all cases, but not for \Hybrid of \Decoupled. While this does add additional computation, the ratio between the extra computation for additional $\beta$ and that for additional $\gamma$ decreases to zero as $\beta$ increases. That is, eventually the cubic scaling in $|\beta|$ will dominate the linear scaling in $|\gamma|$.

\section{Conclusion}

We present a novel orthogonally decoupled basis for variational inference in GP models. Our basis is constructed by extending the standard coupled basis with an additional component to model the mean residues. Therefore, it extends the standard coupled basis~\citep{titsias2009variational,hensman2013gaussian} and achieves better performance. 
We show how the natural parameters of our decoupled basis can be identified and propose an approximate natural gradient update rule, which %
significantly improves the optimization performance over original decoupled approach~\citep{cheng2017variational}. %
Empirically, our method demonstrates strong performance in multiple regression and classification tasks. %

\medskip
{\footnotesize
\bibliographystyle{abbrvnat}
\bibliography{bib}
}
\clearpage
\appendix
\section*{Appendix}

\section{Variational Inference Problem} 
\label{app:variational inference}

In this section, we provides details of implementing the variational inference problem
\begin{align} \tag{\ref{eq:VI problem}}
\LL(q) =  - \sum_{n=1}^{N} \E_{q(f(x_n))}[ \log p(y_n| f(x_n))] + \KL{q}{p}
\end{align}
when the variational posterior $q(f) = \GP_\HH( \mu, \Sigma)$ is parameterized using a decoupled basis
\begin{align} \tag{\ref{eq:general subspace parametrization}}
\mu  = \Psi_\alpha \ab, \qquad 
\Sigma = I + \Psi_\beta \Ab \Psi_\beta^\t.
\end{align}
Without loss of generality, we assume $\Ab = \Kb_\beta^{-1}  \Sb \Kb_\beta^{-1}  - \Kb_\beta^{-1}$. That is, we focus on the following form of parametrization with  $\Sb \succeq 0 $, 
\begin{align} \label{eq:candidate form}
\mu = \Psi_{\alpha}  \ab, 
\qquad 
\Sigma =  I + \Psi_{\beta} \Ab \Psi_{\beta}^\t 
\coloneqq ( I - \Psi_{\beta} \Kb_\beta^{-1} \Psi_{\beta}^\top ) +  \Psi_\beta  \Kb_\beta^{-1} \Sb  \Kb_\beta^{-1} \Psi_\beta^\t. 
\end{align}

\subsection{KL Divergence}

We first show how the KL divergence can be computed using finite-dimensional variables.  The proof is similar to strategy in~\citep[Appendix]{cheng2017variational}
\begin{proposition} \label{pr:KL divergence}
	For $p = \GP_\HH(0, I)$ and $q(f) = \GP_\HH(\mu, \Sigma)$ with
	\begin{align*} 
	\mu = \Psi_{\alpha} \ab, 
	\qquad 
	\Sigma = \left( I - \Psi_{\beta} \Kb_\beta^{-1} \Psi_{\beta}^\top \right) +  \Psi_\beta \Kb_\beta^{-1} \Sb  \Kb_\beta^{-1} \Psi_\beta^\top
	\end{align*}
	Then It satisfies
	\begin{align*}
	\KL{q}{p} &= \frac{1}{2} \left(  \ab^\t \Kb_{{\alpha}} \ab
	+ \tr{ \Sb \Kb_{\beta}^{-1} }  - \log | \Sb |  + \log | \Kb_\beta | - |\beta| \right)
	\end{align*} 
\end{proposition}
For the orthogonally decoupled basis~\eqref{eq:orthogonal RKHS} in particular, we can write 
\begin{align*}
\KL{q}{p} &= \frac{1}{2} \left(  \ab_{\gamma}^\t (\Kb_{\gamma} -\Kb_{\gamma,\beta} \Kb_{\beta}^{-1} \Kb_{\beta,\gamma})  \ab_{\gamma} + \ab_{\beta}^\t \Kb_{\beta} \ab_{\beta}
+ \tr{ \Sb \Kb_{\beta}^{-1} }  - \log | \Sb |  + \log | \Kb_\beta | - |\beta| \right).
\end{align*}

\subsection{Expected Log-Likelihoods}

The expected log-likelihood can be computed by first computing the predictive Gaussian distribution $q(f(x)) = \NN(f(x) | m(x), s(x) )$ for each data point $x$. For example, for the orthogonally decoupled basis~\eqref{eq:orthogonal RKHS}, this is given as 
\begin{align*}%
m(x) &= (\kb_{x,\gamma} - \kb_{x, \beta} \Kb^{-1}_\beta \Kb_{\beta, \gamma} ) \ab_{\gamma} + \kb_{x, \beta} \ab_{\beta}
\\
s(x) &= (k_x - \kb_{x,\beta} \Kb^{-1}_\beta \kb_{\beta, x}) + 
\kb_{x,\beta} \Kb^{-1}_{\beta} \Sb\Kb^{-1}_{\beta} \kb_{\beta, x}.
\end{align*}
Given $m(x)$ and $s(x)$, then the expected log-likelihood can be computed exactly (for Gaussian case) or using quadrature approximation. %

\subsection{Gradient Computation}

Using the above formulas, a differentiable computational graph can be constructed and then the gradient can to $(\ab_{\gamma}, \ab_{\beta}, \Lb)$ can be computed using automatic differentiation. When $ \ab_{\gamma}^\t \Kb_{\gamma} \ab_{\gamma}$ in the KL-divergence is further approximated by column sampling $\Kb_{\gamma}$, an unbiased gradient can be computed in time complexity $O(|\gamma||\beta| + |\beta|^3 )$.

\section{Convexity of the Variational Inference Problem} \label{app:convexity}

Here we show the objective function in~\eqref{eq:VI problem} is strictly convex in $(\ab_{\gamma}, \ab_{\beta}, \Lb)$ if the likelihood is log-strictly-convex. 

\subsection{KL Divergence}
We first study the KL divergence term. It is easy to see that it is strongly convex in $(\ab_{\gamma}, \ab_{\beta})$. When $\Sb = \Lb \Lb^\top$, where $\Lb$ is lower triangle and with positive diagonal terms, the KL divergence is strongly convex in $\Lb$ as well. To see this, we notice that 
\begin{align*}
- \log | \Sb | = - \log | \Lb\Lb^\t | =    -2 \sum_{i=1}^{|\beta|} \log | L_{ii} |
\end{align*}
is strictly convex and %
\begin{align*}
\tr{\Kb_\beta^{-1} \Sb } 
= \tr{ \Lb\Lb^\t \Kb_\beta^{-1}  }
=  \vec{\Lb}^\top ( \Ib \otimes  \Kb_\beta^{-1}   ) \vec{\Lb}
\end{align*}
is strongly convex, because $ \Kb_\beta^{-1} \succ 0 $.

\subsection{Expected Log-likelihood}

Here we show the negative expected log-likelihood part is strictly convex. For the  negative expected log-likelihood, let $F( \cdot ) = - \log(y_n | \cdot )$ and we can write 
\begin{align*}  %
\EE_n &= \E_{q(f(x_n))}[ - \log p(y_n| f(x_n)) ]  \\
&= \E_{\zeta, \bm{\xi} }[ 
F( m(x_n ) + \zeta + \kb_{x_n, \beta}  \Kb_\beta^{-1}  \Lb \bm{\xi}  )   ]
\end{align*}
in which 
$
\zeta \sim \NN(\zeta | 0, k_{x_n} - \kb_{x_n, \beta} \Kb_\beta^{-1} \kb_{\beta, x_n} )$ and $
\bm{\xi} \sim \NN(\bm{\xi} | 0, \Ib )$. 

Then we give a lemma below. 
\begin{lemma} \label{lm:strictly convex preversed by linearity}
	Suppose $f$ is $\theta$-strictly convex. Then $f(Ax)$ is also $\theta$-strictly convex.
\end{lemma}
\begin{proof}
	Let $u  = A x$ and $ v = A y$. Let $g(x) = f(Ax)$. 
	\begin{align*}
	f(v) - f(u) &\geq  \lr{\nabla f(u)}{v - u } + \frac{\theta}{2} ( \lr{\nabla f(u)}{v - u } )^2 \\
	&= \lr{\nabla f(u)}{ A(y-x) } + \frac{\theta}{2} ( \lr{\nabla f(u)}{  A(y-x) } )^2 \\
	&= \lr{A^\top \nabla f(u)}{ y-x } + \frac{\theta}{2} ( \lr{A^\top \nabla f(u)}{ y-x } )^2 \\
	&= \lr{\nabla g(x)}{ y-x } + \frac{\theta}{2} ( \lr{\nabla g(x)}{ y-x } )^2  \qedhere
	\end{align*}
\end{proof}
Because $F$ is strictly convex when likelihood is log-strictly-concave and $m(x_n)$ is linearly parametrized, the desired strict convexity follows.

\section{Uniqueness of Parametrization and Natural Parameters}
\label{app:nat param}

Here we provide some additional details regarding natural parameters and natural gradient descent. 

\subsection{Necessity of Including $\beta$ as Subset of $\alpha$}

We show that the partition condition in Section~\ref{sec:orthogonal} is necessary to derive proper natural parameters. Suppose the contrary case where $\alpha$ is a general set of inducing points. Using a similar derivation as Section~\ref{sec:parameters}, we show that 
\begin{align*}
\frac{1}{2}\Sigma^{-1} &=  \frac{1}{2} \left( I - \Psi_{\beta} \Kb_{\beta}^{-1} \Psi_{\beta}^\t \right) +  \frac{1}{2} \Psi_{\beta} \Sb^{-1} \Psi_{\beta}^\t \\
\Sigma^{-1} \mu  
&=\left( \left( I - \Psi_{\beta} \Kb_{\beta}^{-1} \Psi_{\beta}^\t \right) + \Psi_{\beta} \Sb^{-1} \Psi_{\beta}^\t \right) \Psi_{\alpha} \ab  \\
&= ( I - \Psi_{\beta} \Kb_{\beta}^{-1} \Psi_{\beta}^\t ) \Psi_{\alpha} \tilde\jb_\alpha 
+ \Psi_{\beta} \tilde\jb_\beta  
\end{align*}
where $\tilde\jb_\alpha = \ab_\alpha$ and $\tilde\jb_\beta = \Sb^{-1} \Kb_\beta \ab_\alpha $. Therefore, we might consider choosing $(\jb_\alpha, \jb_\beta, \frac{1}{2}\Sb^{-1})$ as a candidate for natural parameters. 
However the above choice of parametrization is actually coupled due to the condition that  $\jb_\alpha$ and $\jb_\beta$ have to satisfy, i.e.
\begin{align*}
\tilde\jb_\beta = \Sb^{-1} \Kb_\beta \tilde\jb_\alpha
\end{align*} 
Thus, they cannot satisfy the requirement of being natural parameters. This is mainly because $\mu$ is given in only $\alpha$ basis, whereas $\Sigma^{-1} \mu$ is given in both $\alpha$ and $\beta$ bases.

\subsection{Alternate Choices of Natural Parameters}

As discussed previously in Section~\ref{sec:parameters}, the choice of natural parameters is only unique up to affine transformation. While in this paper we propose to use the unique orthogonal version, other choices of parametrization are possible. For instance, here we consider the hybrid parametrization in~\citep[appendix]{cheng2017variational} and give an overview on finding its natural parameters. 

The hybrid parametrization use the following decoupled basis: 
\begin{align*}
\mu = \Psi_\gamma \ab_\gamma + \Psi_\beta \ab_\beta 
\qquad 
\Sigma = (I - \Psi_\beta \Kb^{-1}_\beta \Psi_\beta^\top) + 
\Psi_\beta \Kb^{-1}_{\beta} \Sb\Kb^{-1}_{\beta}\Psi_\beta ^\top
\end{align*}
To facilitate a clear comparison, here we remove the $ \Kb_\beta^{-1}$ in the original  form suggested by~\citet{cheng2017variational}, which uses $\mu = \Psi_\gamma \ab_\gamma + \Psi_\beta  \Kb_\beta^{-1} \ab_\beta $. Note in the experiments, their original form was used.

As the covariance part above is the same form as our orthogonally decoupled basis in~\eqref{eq:orthogonal RKHS}, here we only consider the mean part. Following a similar derivation, we can write 
\begin{align*}
\Sigma^{-1} \mu  
&=\left( \left( I - \Psi_{\beta} \Kb_{\beta}^{-1} \Psi_{\beta}^\t \right) + \Psi_{\beta} \Sb^{-1} \Psi_{\beta}^\t \right) (\Psi_{\gamma} \ab_\gamma + \Psi_{\beta} \ab_\beta)  \\
&=  \Psi_{\gamma} \ab_\gamma + \Psi_\beta (\Sb^{-1} - \Kb_\beta^{-1} ) \Kb_{\beta, \gamma} \ab_\gamma + \Psi_{\beta} \Sb^{-1}  \Kb_\beta \ab_\beta  \\
&=  (\Psi_{\gamma}   - \Psi_\beta \Kb_\beta^{-1} \Kb_{\beta, \gamma} ) \ab_\gamma    +
\Psi_{\beta} \Sb^{-1} ( \Kb_\beta \ab_\beta +   \Kb_{\beta, \gamma} \ab_\gamma ) \\
&=  (I  - \Psi_\beta \Kb_\beta^{-1} \Psi_{\beta}^\t ) \Psi_{\gamma}  \jb_\gamma   + \Psi_{\beta} \jb_\beta 
\end{align*}
That is, we can choose the natural parameters as 
\begin{align} \label{eq:alternate natural paramter example}
\jb_\gamma = \ab_\gamma, 
\qquad 
\jb_\beta = \Sb^{-1} ( \Kb_\beta \ab_\beta +   \Kb_{\beta, \gamma} \ab_\gamma ), 
\qquad 
\bm{\Theta} = \frac{1}{2} \Sb^{-1}
\end{align}
This set of natural parameters, unlike the one in the previous section, is proper, because $\beta$ included as a subset of $\alpha$. 

Comparing~\eqref{eq:alternate natural paramter example} with~\eqref{eq:natural parameters}, we can see that there is a coupling between $\ab_{\gamma}$ and $\jb_\beta$ in \eqref{eq:alternate natural paramter example}. This would lead to a more complicated update rule in computing the natural gradient. 
This coupling phenomenon also applies to other choice of parametrizations, excerpt for our orthogonally decoupled basis.

\subsection{Invariance of Natural Gradient Descent}

As discussed above, the choice of natural parameters for the mean part is not unique, but here we show they all lead to the same natural gradient descent update. Therefore, our orthogonal choice~\eqref{eq:orthogonal RKHS}, among all possible equivalent parameterizations, has the cleanest update rule. 

This equivalence between different parameterizations can be easily seen from that the KL divergence between Gaussians are quadratic in $\Sigma^{-1} \mu$. Therefore,  the natural gradient of  $\Sigma^{-1} \mu$ has the form as the proximal update below
\begin{align*}
\argmin_x \lr{\nabla_x f}{x} + \frac{1}{2} (x - y)^\t Q (x-y) =  y - Q^{-1} g
\end{align*}
for some function $f$, vector $y$ and positive-definite matrix $Q$.

To see the invariance of invertible linear transformations, suppose we reparametrize $x,y$ above as $x = A u + b$ and $y =  A v + b$,
for some invertible $A$ and $b$.
Then the update becomes 
\begin{align*}
&\argmin_u \lr{\nabla_z f}{u} + \frac{1}{2} (u - v)^\t A^\t Q A (u - v) \\
&\argmin_u \lr{ A^\t \nabla_x f}{ u} + \frac{1}{2} (u - v)^\t A^\t Q A (u - v)\\
&=  v - A^{-1} Q^{-1}  \nabla_x f
\end{align*}
which represents the same update step in $x$ because
\begin{align*}
A( v - A^{-1} Q^{-1}  \nabla_x f) + b =  y - Q^{-1}  \nabla_x f.
\end{align*}

\subsection{Transformation of Natural Parameters and Expectation Parameters}

Here we provide a more rigorous proof of identifying natural and expectation parameters of decoupled bases, as the density function $p(f)$, which is used to illustrate the idea in Section~\ref{sec:parameters}, is not defined for GPs.
Here we show the transformation of natural parameters and expectation parameters based on KL divergence. We start from  $d$-dimensional exponential families and then show that the formulation extends to arbitrary $d$ .

Consider a $d$-dimensional exponential family. Its KL divergence of an exponential family can be written as 
\begin{align} \label{eq:KL in primal and dual}
\KL{q}{p} = A(\eta_p) + A^*(\theta_q) - \lr{\eta_p}{\theta_q}
\end{align}
where $A$ is the log-partition function, $A^*$ is its Legendre dual of $A$, $\theta$ is the expectation parameter, and $\eta$ is the natural parameter. It holds the duality property that $\theta_p = \nabla A(\eta_p)$ and  $\eta_p = \nabla A^*(\theta_p)$.

As~\eqref{eq:KL in primal and dual} is expressed in terms of inner product, it holds for arbitrary $d$ and it is defined finitely for GPs with decoupled basis~\citep{cheng2017variational}. 
Therefore, here we show that when we parametrize problem by $\eta_p = H \tilde{\eta}_p + b$, $\tilde{\eta}_p$ is also a candidate natural parameter satisfying ~\eqref{eq:KL in primal and dual} for some transformed expectation parameter $\tilde{\theta}_q$. It can be shown as below
\begin{align*}
\KL{q}{p} 
&=  A(\eta_p) + A^*(\theta_q) - \lr{\eta_p}{\theta_q}\\
&= A(H \tilde{\eta}_p + b) + A^*(\theta_q) - \lr{ H \tilde{\eta}_p + b}{\theta_q} \\
&= A( H \tilde{\eta}_p + b ) + A^*(\theta_q) - \lr{ \tilde{\eta}_p}{  H^\t  \theta_q} - \lr{b}{\theta_q} \\
&= A( H\tilde{\eta}_p + b ) + \left(A^*(H^{-\t} \tilde{\theta}_q) - \lr{b}{H^{-\t} \tilde{\theta}_q}\right) -  \lr{ \tilde{\eta}_p}{  \tilde{\theta}_q}   \\
&\eqqcolon \tilde{A}(\tilde{\eta}_p) + \tilde{A}^*(\tilde{\theta}_q)
-  \lr{ \tilde{\eta}_p}{  \tilde{\theta}_h} 
\end{align*}
where we define
\begin{align*}
\tilde{\theta}_q &= H^\t \theta_q \\
\tilde{A}(\tilde{\eta}_p) &=  A( H\tilde{\eta}_p + b ) \\
\tilde{A}^*(\tilde{\theta}_q) &= A^*(H^{-\t} \tilde{\theta}_q) - \lr{b}{H^{-\t} \tilde{\theta}_q} 
\end{align*}
It can be verified that $\tilde{A}^*$ is indeed the Legendre dual of $\tilde{A}$. 
\begin{align*}
\max_x \lr{w}{x} - \tilde{A}(x)  
&= \max_x \lr{w}{x} - A(Hx + b)\\
&= \max_z \lr{w}{H^{-1}(z-b)} - A(z) \\
&= -\lr{H^{-\t} w}{b} + \max_z \lr{H^{-\t} w}{z} - A(z)\\
&= -\lr{H^{-\t} w}{b} + A^*(H^{-\t} w ) = \tilde{A}^*(w) 
\end{align*}
Note the inversion requirement on $H$ can be removed  by replacing $-\t$ with pseudo-inverse, because $ \tilde{\theta}_q $ lies in the range of $H^\t$. 
Thus, if  $\eta = H \tilde{\eta} + b$ and $\theta$ are one choice of natural-expectation parameter pair, then $\tilde\eta$ and $\tilde\theta = H^T \eta$ is another natural-expectation parameter pair.

\section{Primal Representation of Orthogonally Decoupled GPs}
\label{app:inducing_point_construction}
In this section, we demonstrate that the orthogonally decoupled GPs have an equivalent construction from the primal viewpoint adopted by variational inference framework of~\citet{titsias2009variational}. 
The key idea is to use two sets of inducing points. We use them to form a posterior process by conditioning the prior like the usual way, but in the meantime imposing a particular restriction on the variational distribution at the inducing points. %

\subsection{The Variational Posterior Process Proposed by~\citet{titsias2009variational}}
The approach of~\citet{titsias2009variational} begins with expressing the prior process in terms of the following factorization\footnote{We follow the conventional abuse of notation by writing the process as if it has a density. See~\citet{matthews2017thesis} for a rigorous treatment that defines the posterior processes as in terms of Radon-Nikodym derivative with respect to the prior.}:
\begin{align*}
p(f) =& p(f | \mbf{f}_\beta)p(\mbf{f}_\beta)\,,
\end{align*}
where $\mbf{f}_\beta$ are function values at locations $\beta$, often referred to as ``inducing points.'' For simplicity we assume zero prior mean, so the prior at the inducing points is $p(\mbf{f}_\beta) = \mathcal{N}(\mbf{f}_\beta|\mbf{0}, \mbf{K}_\beta)$ and the prior conditional process $p(f | \mbf{f}_\beta)$ is a GP which we denote  as $\mathcal{GP}(m_{\mbf{f}_\beta}, s_{\mbf{f}_\beta})$ with
\begin{align*}
m_{\mbf{f}_\beta}(x) =& \mbf{k}_{x,\beta}^{\top} \mbf{K}_{\beta}^{-1}\mbf{f}_{\beta}\\
s_{\mbf{f}_\beta}(x, x') =& k(x, x') -  \mbf{k}_{x, \beta}^{\top} \mbf{K}_{\beta}^{-1} \mbf{k}_{\beta, x'}
\end{align*}

The key idea of~\citet{titsias2009variational}, which is later developed by~\citep{hensman2013gaussian, matthews2017thesis}, is to define the variational posterior process as 
\begin{align}\label{eq:titsias_assumption}
q(f) = p(f | \mbf{f}_\beta)q(\mbf{f}_\beta)\,,
\end{align}
where $q(\mbf{f}_\beta)=\mathcal{N}(\mbf{f}_\beta|\mbf{m}_\beta, \mbf{S}_\beta)$ for some variational parameters $\mbf{m}_\beta$ and $\mbf{S}_\beta$.  %
Since the conditional process is linear in $\mbf{f}_\beta$ and $q(\mbf{f}_\beta)$ is Gaussian, we can use standard properties for Gaussians (i.e., $\int_x \mathcal{N}(y | a+Lx, A)\mathcal{N}(x |b, B) dx\propto \mathcal{N}(y | a+Lb, A+L B L^\t)$) to derive the mean and covariance functions of the variational posterior process $q(f)$ in~\eqref{eq:titsias_assumption}:
\begin{align}
m(x) =& \mbf{k}_{x, \beta}\mbf{K}_{\beta}^{-1}\mbf{m}_{\beta}\\
s(x, x') =& k(x, x') +  \mbf{k}_{x, \beta}\mbf{K}_{\beta}^{-1} (\mbf{S}_\beta-\mbf{K}_{\beta})\mbf{K}_{\beta}^{-1}\mbf{k}_{\beta, x'}
\end{align}

\subsection{The Equivalent Posterior Process of the Orthogonally Decoupled Basis}
To derive our orthogonally decoupled approach, we introduce further a set of disjoint inducing points denoted as $\gamma$. Let $\mbf{f}_\gamma$ be the function values at locations $\gamma$. The prior process can be expressed as
\begin{align} \label{eq:prior conditional distribution of gamma on beta}
p(f) =& p(f | \mbf{f}_\gamma, \mbf{f}_\beta)p(\mbf{f}_\gamma | \mbf{f}_\beta) p(\mbf{f_\beta}),
\end{align}
where $p(\mbf{f}_\beta)$ is defined as before, the prior conditional distribution of $\mbf{f}_\gamma $ given $\mbf{f}_\beta$ can be written as 
\begin{align*}
p(\mbf{f}_\gamma | \mbf{f}_\beta) = \mathcal{N}(
\mbf{K}_{\gamma, \beta}\mbf{K}_{\beta}^{-1}\mbf{f}_{\beta}, 
\quad 
\mbf{K}_\gamma -  \mbf{K}_{\gamma,\beta}\mbf{K}_{\beta}^{-1} \mbf{K}_{\beta,\gamma}
),
\end{align*}
and $p(f | \mbf{f}_\gamma, \mbf{f}_\beta)$ the prior conditional process conditioned on $\mbf{f}_\gamma$ and $\mbf{f}_\beta$ is a GP, which we denote as $\mathcal{GP}(m_{\mbf{f}_{\gamma}, \mbf{f}_{\beta}}, s_{\mbf{f}_{\gamma}, \mbf{f}_{\beta}})$ and has the following mean and covariance functions
\begin{align}
\label{eq:posterior_mean_decoupled}
m_{\mbf{f}_{\gamma}, \mbf{f}_{\beta}}(x) = & 
\begin{bmatrix}
\mbf{k}_{x, \gamma} &
\mbf{k}_{x, \beta}
\end{bmatrix} 
\begin{bmatrix}
\mbf{K}_{\gamma} & \mbf{K}_{\gamma,\beta} \\
\mbf{K}_{\beta,\gamma} & \mbf{K}_{\beta}
\end{bmatrix}^{-1}
\begin{bmatrix}
\mbf{f}_{\gamma} \\
\mbf{f}_{\beta}
\end{bmatrix}\\
\label{eq:posterior_cov_decoupled}
s_{\mbf{f}_{\gamma}, \mbf{f}_{\beta}}(x, x') = &
k(x, x') -
\begin{bmatrix}
\mbf{k}_{x, \gamma} &
\mbf{k}_{x, \beta} 
\end{bmatrix}
\begin{bmatrix}
\mbf{K}_{\gamma} & \mbf{K}_{\gamma,\beta} \\
\mbf{K}_{\beta,\gamma} & \mbf{K}_{\beta}
\end{bmatrix}^{-1}
\begin{bmatrix}
\mbf{k}_{\gamma, x} \\
\mbf{k}_{\beta, x} 
\end{bmatrix}
\end{align}

Following the same idea of~\citet{titsias2009variational}, we consider a variational posterior written as 
\begin{align}\label{eq:decoupled_inducing_construction}
q(f) =& p(f | \mbf{f}_\gamma, \mbf{f}_\beta)q(\mbf{f}_\gamma, \mbf{f_\beta}).
\end{align}

Now we show how to parameterize $q(\mbf{f}_\gamma, \mbf{f_\beta})$ so that~\eqref{eq:decoupled_inducing_construction} defines an orthogonally decoupled GP. 
Note that if we parameterized this distribution as a full-rank Gaussian with no further restriction, it would be equivalent to just absorbing $\gamma$ into $\beta$ and would incur the computational complexity that we seek to avoid. 

To obtain an orthogonally decoupled posterior, we use the form
\begin{align}\label{eq:factorised_q_two_sets}
q(\mbf{f}_\gamma, \mbf{f}_\beta ) = q(\mbf{f}_\gamma | \mbf{f}_\beta)q(\mbf{f}_\beta)\,,
\end{align}
where $q(\mbf{f_\beta})=\mathcal{N}(\mbf{m}_\beta, \mbf{S}_\beta)$, and we define
\begin{align}\label{eq:q_f_gamma_given_beta}
q(\mbf{f}_\gamma | \mbf{f}_\beta) = 
\mathcal{N}(
\mbf{m}_{\gamma \perp \beta}
 +  \mbf{K}_{\gamma,\beta}\mbf{K}_{\beta}^{-1} \mbf{f}_\beta , 
\quad
\mbf{K}_\gamma -  \mbf{K}_{\gamma,\beta}\mbf{K}_{\beta}^{-1} \mbf{K}_{\beta,\gamma})
\end{align}
for some variational parameter $\mbf{m}_{\gamma \perp \beta}$.
Note that $q(\mbf{f}_\gamma | \mbf{f}_\beta)$ is a Gaussian distribution that matches $p(\mbf{f}_\gamma | \mbf{f}_\beta)$ in~\eqref{eq:prior conditional distribution of gamma on beta} in covariance, but does \emph{not} match in the mean unless $\mbf{m}_{\gamma \perp \beta}=0$. If we were to set %
$q(\mbf{f}_\gamma | \mbf{f}_\beta)= p(\mbf{f}_\gamma | \mbf{f}_\beta)$ we would recover the standard result using $\beta$ alone. This is because we would have effectively absorbed $\mbf{f}_\gamma$ into the  prior conditional process. 

Since our choice for $q(\mbf{f}_\gamma | \mbf{f}_\beta)$ matches the prior in the covariance and has the same linear dependency on $\mbf{f}_\beta$, the posterior process of $q(f)$ in~\eqref{eq:decoupled_inducing_construction} has a covariance function as  \eqref{eq:posterior_cov_decoupled}. 
To find its mean function,
let us first write $q(\mbf{f_\gamma}, \mbf{f}_\beta)$ as a joint distribution:
\begin{align*}
q\left(
\begin{bmatrix}
\mbf{f}_{\gamma} \\
\mbf{f}_{\beta}
\end{bmatrix} 
\right)
=
\mathcal N \left(
\begin{bmatrix}
\mbf{m}_{\gamma \perp \beta} +  \mbf{K}_{\gamma,\beta}\mbf{K}_{\beta}^{-1} \mbf{m}_\beta \\
\mbf{m}_{\beta}
\end{bmatrix} 
, 
\begin{bmatrix}
\mbf{K}_{\gamma} + \mbf{K}_{\gamma,\beta}\mbf{K}_{\beta}^{-1} (\mbf{S}_{\beta} - \mbf{K}_{\beta} )\mbf{K}_{\beta}^{-1} \mbf{K}_{\beta,\gamma}& \mbf{K}_{\gamma,\beta}\mbf{K}_{\beta}^{-1} \mbf{S}_{\beta} \\
\mbf{S}_{\beta}\mbf{K}_{\beta}^{-1} \mbf{K}_{\beta,\gamma} & \mbf{S}_{\beta}
\end{bmatrix} 
\right)
\end{align*}
We then can derive the posterior process mean function as
\begin{align}
m(x) = 
\begin{bmatrix}
\mbf{k}_{x, \gamma} &
\mbf{k}_{x, \beta}
\end{bmatrix} 
\begin{bmatrix}
\mbf{K}_{\gamma} & \mbf{K}_{\gamma,\beta} \\
\mbf{K}_{\beta,\gamma} & \mbf{K}_{\beta}
\end{bmatrix}^{-1}
\begin{bmatrix}
\mbf{m}_{\gamma \perp \beta} + 
\mbf{K}_{\gamma,\beta}\mbf{K}_{\beta}^{-1} \mbf{m}_\beta  \\
\mbf{m}_{\beta}
\end{bmatrix}
\end{align}

To simplify the above expression, we write the inverse block matrix explicitly as
\begin{align}
\begin{bmatrix}
\mbf{K}_{\gamma} & \mbf{K}_{\gamma,\beta} \\
\mbf{K}_{\beta,\gamma} & \mbf{K}_{\beta}
\end{bmatrix}^{-1}
=
\begin{bmatrix}
\mbf{K}_{\gamma \perp \beta}^{-1} & -\mbf{K}_{\gamma \perp \beta}^{-1}\mbf{K}_{\gamma,\beta} \mbf{K}_{\beta}^{-1}\\
-\mbf{K}_{\beta}^{-1}\mbf{K}_{\beta,\gamma}\mbf{K}_{\gamma \perp \beta}^{-1}  & \mbf{K}_{\beta}^{-1} +\mbf{K}_{\beta}^{-1}\mbf{K}_{\beta,\gamma}\mbf{K}_{\gamma \perp \beta}^{-1} \mbf{K}_{\gamma,\beta}\mbf{K}_{\beta}^{-1}
\end{bmatrix}
\end{align}
where we define
\begin{align}
\mbf{K}_{\gamma \perp \beta} = \mbf{K}_{\gamma} - \mbf{K}_{\beta,\gamma} \mbf{K}_{\beta}^{-1} \mbf{K}_{\beta,\gamma}.
\end{align}
After canceling several terms, we arrive at the expression
\begin{align}
m(x) = 
\mbf{k}_{x, \gamma} \mbf{K}_{\gamma \perp \beta}^{-1} \mbf{m}_{\gamma \perp \beta}
-\mbf{k}_{x, \beta} \mbf{K}_{\beta,\gamma}\mbf{K}_{\gamma \perp \beta}^{-1} \mbf{m}_{\gamma \perp \beta}
+
\mbf{K}_{x,\beta}\mbf{K}_{\beta}^{-1} \mbf{m}_\beta 
\end{align}

A natural choice is to define $\mbf{a}_\gamma=\mbf{K}_{\gamma \perp \beta}^{-1}\mbf{m}_{\gamma \perp \beta}$ (which agrees with the definition in Figure~\ref{fig:relationship}). In this case, we obtain 
\begin{align}
m(x) = 
(\mbf{k}_{x, \gamma} 
-\mbf{k}_{x, \beta} \mbf{K}_{\beta,\gamma}) \mbf{a_\gamma}
+
\mbf{K}_{x,\beta}\mbf{K}_{\beta}^{-1} \mbf{m}_\beta.
\end{align}
which is exactly the result for the orthogonally decoupled basis, as $\mbf{m}_\beta = \mbf{K}^{-1}_\beta \mbf{a}_\beta$.

The decoupled basis can therefore be interpreted from the inducing perspective as a special case of a structured covariance. The key idea above is to use prior conditional matching, just as in the approach of \citet{titsias2009variational}, but to match only the covariance and not the mean.

\section{Expression for the Optimal Variational Parameters in Decoupled Bases}
\label{sec:analytic_expressions}
In the case of the Gaussian likelihood we can solve the variational inference problem \ref{eq:VI problem} analytically, although doing so incurs a cost that scales cubically in $\alpha$ and prohibits the use of minibatches. 

To make the results mirror the familiar expression for the optimal variational parameters in the coupled case \citep{titsias2009variational}, we use the basis
\begin{align*}
\mu = \Psi_{\alpha} \ab,
\qquad 
\Sigma = I + \Psi_{\beta} \mbf{K}_{\beta}^{-1} ( \Sb - \mbf{K}_{\beta}) \mbf{K}_{\beta}^{-1}  \Psi_{\beta}^\t \,,
\end{align*}
This basis is equivalent to the \textsc{Hybrid} and \textsc{Decoupled} bases through redefinition of parameters. The solution for $\Sb$ is exactly the same as in the coupled case:
\begin{align*}
\Sb&=\left(\frac{1}{\sigma^2}\Kb_{\beta}^{-1}\Kb_{\beta X}\Kb_{X\beta }\Kb_{\beta}^{-1} +  \Kb_{\beta}^{-1}\right)^{-1}\\
&=\Kb_{\beta} \left(\frac{1}{\sigma^2}\Kb_{\beta X}\Kb_{X\beta } +  \Kb_{\beta}\right)^{-1} \Kb_{\beta}\,.
\end{align*}
For $\ab$, we have
\begin{align*}
\ab%
&=   \left(\frac{1}{\sigma^2}\Kb_{\alpha X}\Kb_{X\alpha } +  \Kb_{\alpha}\right)^{-1} \Kb_{\alpha X}\yb 
\end{align*}

\section{Experimental Details}
\label{sec:experimental details}

In our experiments we use sensible defaults and do not hand tune for specific datasets. The full details are as follows:

\paragraph{Kernel} We use the sum of a Matern52 kernel with lengthscale $0.1\sqrt{D}$  and an RBF kernel with lengthscale $\sqrt{D}$, where $D$ is the input dimension. Both kernels are intialized to unit amplitude for regression and amplitude 5 for classification. 

\paragraph{Inducing point initalizations}
We use kmeans to initialize $\beta$ and use a random sample of the data for $\gamma$. We take care to use the same random seeds to ensure consistency between methods. For the \Decoupled basis $\alpha$ we concatenate $\gamma$ and $\beta$ for a fair comparison with the other methods. 

\paragraph{Data preprocessing and splits} The datasets we used had already been preprocessed to have zero mean and unit standard deviation. We construct test sets with a random 10\% split. The splits are the same, so the results are directly comparable between our methods. The results from \citet{klambauer2017self} used a different split from ours, however.

\paragraph{Variational Parameter Initializations} We initialize the variational parameters to the prior. I.e. zero mean and $\Sb=\Kb_\beta$ (NB the $\Bb$ in the \textsc{Decoupled} basis is initialized to near zero). 

\paragraph{Optimization} We the adam optimizer with the default settings in the tensorflow implementation (including a learning rate of 0.001) for 20000 iterations. We use a step size of 0.005 for the natural gradient updates. For the non-conjugate likelihoods we increase from $10^{-5}$ to 0.005 linearly over the first 100 iterations, following the suggestion in \citet{salimbeni2018natural}.

\paragraph{Likelihood} We initialize the Gaussian likelihood variance to 0.1. 

\paragraph{Minibatches} We use a batch size of 1024 for data sub-sampling, and a batch of size 64 for the sub-sampling the columns of the ${\ab_\gamma}^{\top}\Kb_{\gamma}\ab_\gamma$ term in the ELBO.

We implemented all our methods in tensorflow, using on an open-source Gaussian process package, GPflow \cite{2017GPflow}. Our code \footnote{\href{https://github.com/hughsalimbeni/orth\_decoupled\_var\_gps}{\color{blue}\texttt{https://github.com/hughsalimbeni/orth\_decoupled\_var\_gps}}} and datasets \footnote{\href{https://github.com/hughsalimbeni/bayesian_benchmarks}{\color{blue}\texttt{https://github.com/hughsalimbeni/bayesian\_benchmarks}}} are publicly available.

\section{Further Results}
\label{sec:further_results}

\begin{figure*}[h!]
    \centering
    \begin{subfigure}{0.5\textwidth}
        \centering
        \includegraphics[width=\textwidth]{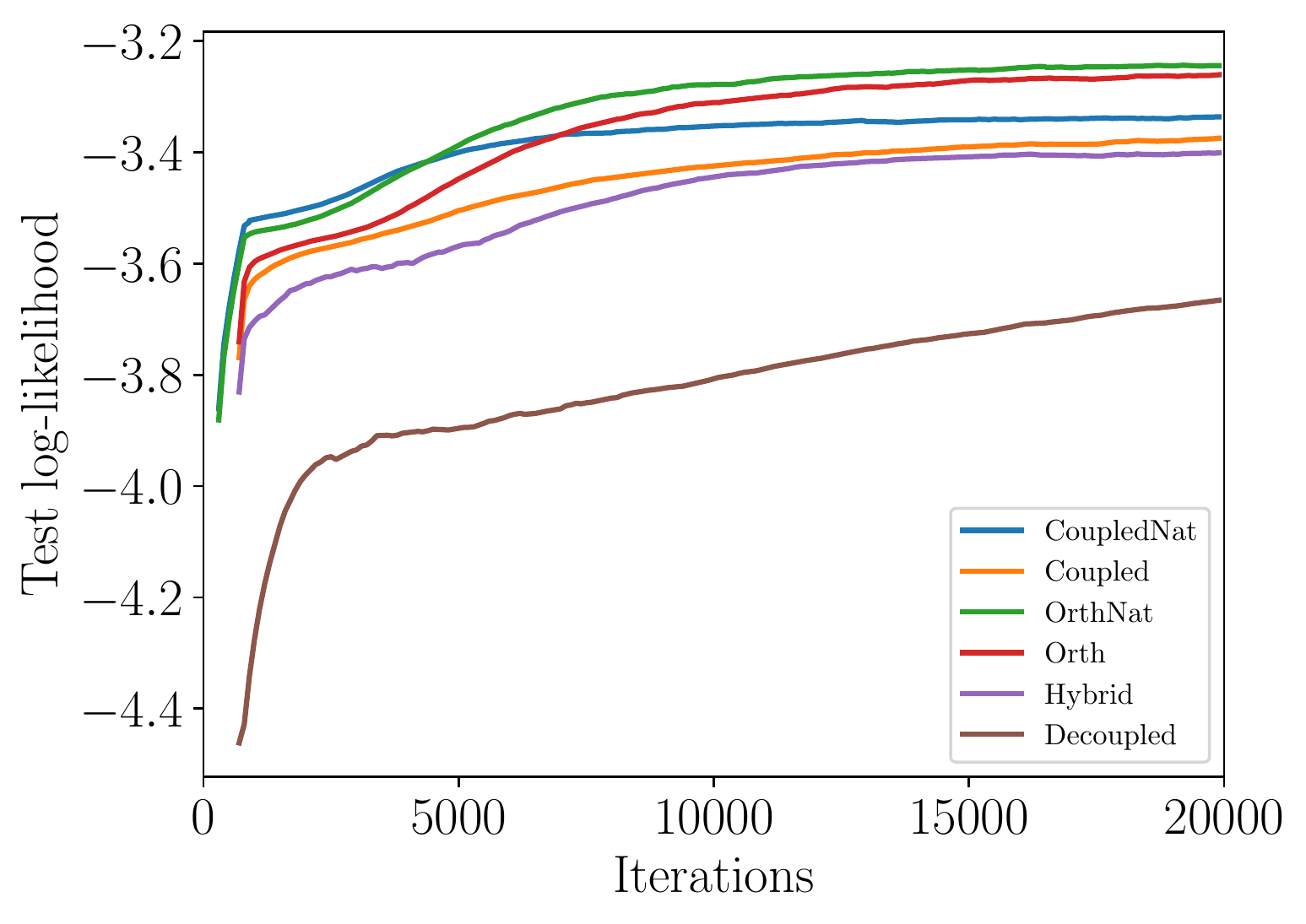}
        \caption{Test log-likelihood}
        \label{fig:large_scale_lik}
    \end{subfigure}%
    \begin{subfigure}{0.5\textwidth}
        \centering
        \includegraphics[width=\textwidth]{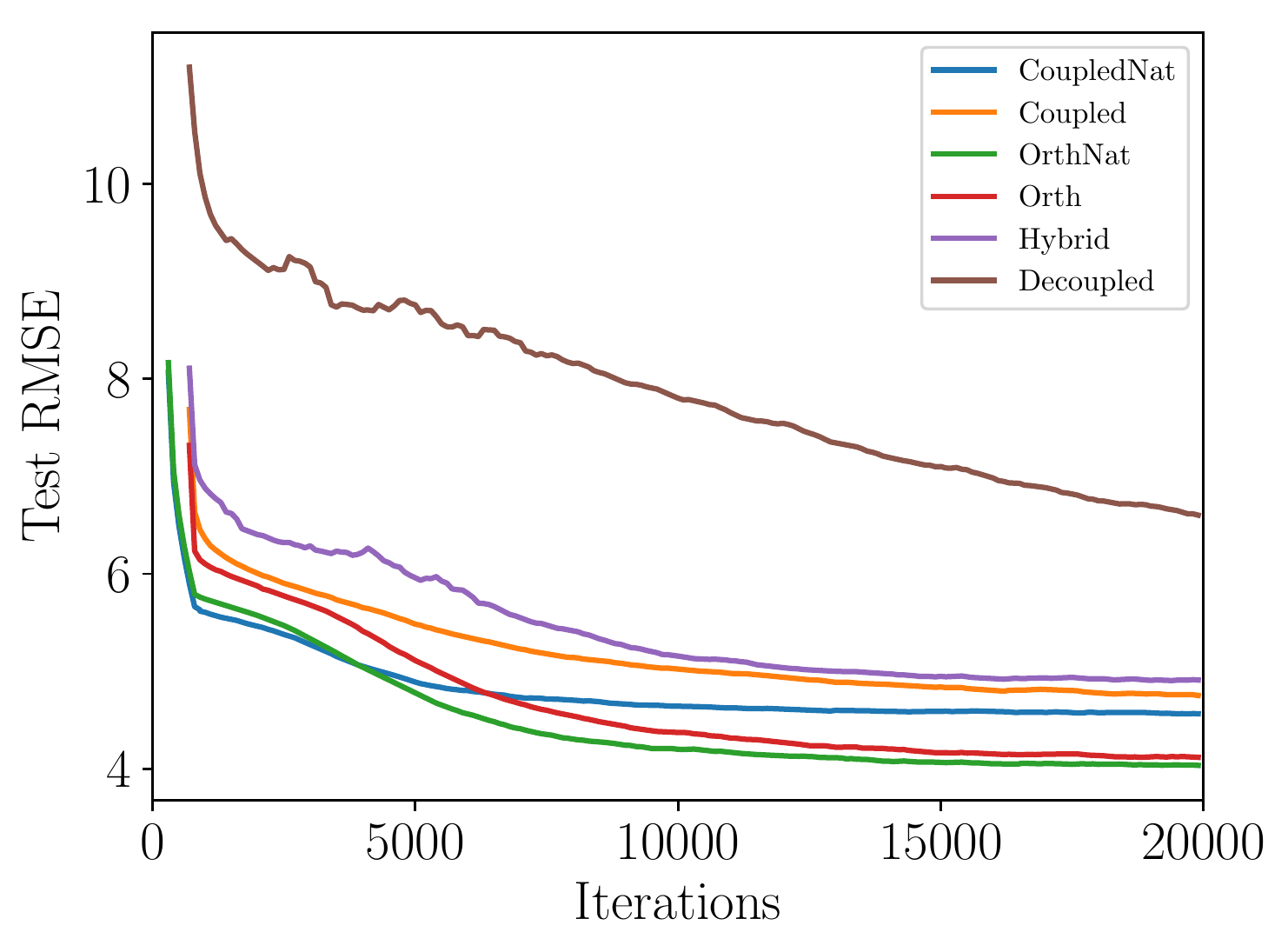}
        \caption{Test MAE}
    \end{subfigure}
    \caption{Test log-likelihood (a), and accuracy (b) for the large scale experiment. The ELBO is reported in the main text, Figure~\ref{fig:elbo}}
\label{fig:extra_plots}
\end{figure*}

\begin{table}
\resizebox{\columnwidth}{!}{%
\begin{tabular}{lllllllllll}
\toprule
                &        N &    D & \textsc{Coupled}$\dagger$ & \textsc{CoupledNat}$\dagger$ & \textsc{Coupled} & \textsc{CoupledNat} &  \textsc{OrthNat} & \textsc{Orth} & \textsc{Hybrid} & \textsc{Decoupled} \\
\midrule
         3droad &   434874 &    3 &                   -0.7630 &                      -0.7632 &          -0.7218 &             -0.7228 &  \textbf{-0.5947} &       -0.6103 &         -0.7617 &            -0.9438 \\
  houseelectric &  2049280 &   11 &                    1.3130 &                       1.3563 &           1.3383 &              1.3727 &   \textbf{1.3899} &        1.3719 &          1.3092 &             0.6032 \\
          slice &    53500 &  385 &                    0.7816 &                       0.7868 &           0.8321 &              0.8415 &   \textbf{0.8776} &        0.8701 &          0.7852 &             0.0655 \\
      elevators &    16599 &   18 &                   -0.4475 &                      -0.4455 &          -0.4448 &    \textbf{-0.4438} &           -0.4479 &       -0.4441 &         -0.4585 &            -0.4966 \\
           bike &    17379 &   17 &                    0.0059 &                       0.0135 &           0.0321 &     \textbf{0.0419} &            0.0271 &        0.0317 &         -0.0318 &            -0.1783 \\
   keggdirected &    48827 &   20 &                    1.0134 &                       1.0158 &           1.0214 &              1.0223 &   \textbf{1.0224} &        1.0216 &          1.0102 &             0.8947 \\
            pol &    15000 &   26 &                    0.0726 &                       0.0821 &           0.1047 &              0.1132 &   \textbf{0.1586} &        0.1451 &          0.0784 &            -0.2502 \\
 keggundirected &    63608 &   27 &                    0.6984 &                       0.6999 &           0.6994 &     \textbf{0.7020} &            0.7007 &        0.6967 &          0.6878 &             0.6374 \\
        protein &    45730 &    9 &                   -0.9531 &                      -0.9535 &          -0.9375 &             -0.9361 &  \textbf{-0.9138} &       -0.9165 &         -0.9527 &            -1.0464 \\
           song &   515345 &   90 &                   -1.1902 &                      -1.1898 &          -1.1890 &             -1.1884 &  \textbf{-1.1880} &       -1.1882 &         -1.1909 &            -1.2266 \\
           buzz &   583250 &   77 &                   -0.0566 &                      -0.0551 &          -0.0512 &             -0.0490 &  \textbf{-0.0480} &       -0.0484 &         -0.0614 &            -0.2285 \\
         kin40k &    40000 &    8 &                    0.0561 &                       0.1580 &           0.2191 &     \textbf{0.2234} &            0.1931 &        0.1777 &          0.1531 &            -0.3877 \\
           Mean &          &      &                    0.0442 &                       0.0588 &           0.0752 &              0.0814 &   \textbf{0.0981} &        0.0923 &          0.0472 &            -0.2131 \\
         Median &          &      &                     0.031 &                        0.048 &            0.068 &               0.078 &    \textbf{0.093} &         0.088 &           0.023 &             -0.239 \\
       Avg Rank &          &      &                     2.917 &                        4.000 &            5.417 &               6.750 &    \textbf{7.083} &         6.250 &           2.583 &              1.000 \\
\bottomrule
\end{tabular}

}
\caption{Regression results normalized test likelihoods. High numbers are better. The coupled bases had $|\beta|=400$ ($|\beta|=300$ for the $\dagger$ bases), and the decoupled all had $\gamma=700$, $\beta=300$. We note that the orthogonal bases always outperform their coupled counterparts with the same $\beta$, but this does not hold for the \textsc{Decoupled} or \textsc{Hybrid} bases} 
\label{table:regression_lik_full}
\end{table}

\begin{table}
\label{table:regression_rmse_full}
\resizebox{\columnwidth}{!}{%
\begin{tabular}{lllllllllll}
\toprule
                &        N &    D & \textsc{Coupled}$\dagger$ & \textsc{CoupledNat}$\dagger$ & \textsc{Coupled} & \textsc{CoupledNat} & \textsc{OrthNat} &    \textsc{Orth} & \textsc{Hybrid} & \textsc{Decoupled} \\
\midrule
         3droad &   434874 &    3 &                    0.5166 &                       0.5163 &           0.4946 &              0.4951 &  \textbf{0.4332} &           0.4395 &          0.6179 &             0.5150 \\
  houseelectric &  2049280 &   11 &                    0.0639 &                       0.0611 &           0.0615 &              0.0595 &  \textbf{0.0583} &           0.0594 &          0.1286 &             0.0636 \\
          slice &    53500 &  385 &                    0.0840 &                       0.0848 &           0.0787 &              0.0779 &  \textbf{0.0730} &           0.0736 &          0.2112 &             0.0838 \\
      elevators &    16599 &   18 &                    0.3767 &                       0.3760 &           0.3756 &              0.3753 &           0.3770 &  \textbf{0.3752} &          0.3973 &             0.3812 \\
           bike &    17379 &   17 &                    0.2342 &                       0.2324 &           0.2283 &     \textbf{0.2261} &           0.2293 &           0.2282 &          0.2848 &             0.2438 \\
   keggdirected &    48827 &   20 &                    0.0883 &                       0.0878 &           0.0874 &     \textbf{0.0871} &           0.0871 &           0.0873 &          0.0980 &             0.0883 \\
            pol &    15000 &   26 &                    0.2073 &                       0.2059 &           0.2012 &              0.2000 &  \textbf{0.1906} &           0.1931 &          0.2867 &             0.2065 \\
 keggundirected &    63608 &   27 &                    0.1200 &                       0.1196 &           0.1197 &     \textbf{0.1191} &           0.1194 &           0.1202 &          0.1304 &             0.1223 \\
        protein &    45730 &    9 &                    0.6207 &                       0.6216 &           0.6118 &              0.6113 &  \textbf{0.5963} &           0.5972 &          0.6868 &             0.6209 \\
           song &   515345 &   90 &                    0.7954 &                       0.7952 &           0.7944 &              0.7939 &  \textbf{0.7936} &           0.7938 &          0.8275 &             0.7960 \\
           buzz &   583250 &   77 &                    0.2601 &                       0.2606 &           0.2586 &              0.2579 &  \textbf{0.2574} &           0.2576 &          0.3106 &             0.2617 \\
         kin40k &    40000 &    8 &                    0.2087 &                       0.1885 &           0.1768 &              0.1746 &  \textbf{0.1740} &           0.1776 &          0.3501 &             0.1887 \\
\midrule
Mean &          &      &                    0.2980 &                       0.2958 &           0.2907 &              0.2898 &  \textbf{0.2824} &           0.2836 &          0.3608 &             0.2977 \\
         Median &          &      &                     0.221 &                        0.219 &            0.215 &               0.213 &   \textbf{0.210} &            0.211 &           0.299 &              0.225 \\
       Avg Rank &          &      &                     6.083 &                        5.167 &            3.750 &               2.417 &   \textbf{1.833} &            2.500 &           8.000 &              6.250 \\
\bottomrule
\end{tabular}
}
\caption{As Table \ref{table:regression_lik_full} but reporting test RMSE. Lower numbers are better.}
\end{table}

\begin{table}
\resizebox{\columnwidth}{!}{%
\begin{tabular}{lllllllllll}
\toprule
                    &       N &   D &   K &             Selu & \textsc{Coupled} & \textsc{CoupledNat} &    \textsc{Orth} & \textsc{OrthNat} &  \textsc{Hybrid} & \textsc{Decoupled} \\
\midrule
              adult &   48842 &  15 &   2 &            84.76 &            85.85 &      \textbf{86.11} &            85.65 &   \textbf{86.15} &            85.73 &              84.37 \\
         chess-krvk &   28056 &   7 &  18 &   \textbf{88.05} &            67.38 &               60.23 &            67.76 &            60.70 &            59.34 &              53.56 \\
          connect-4 &   67557 &  43 &   2 &   \textbf{88.07} &            85.54 &               86.44 &            85.99 &            86.33 &            85.13 &              83.12 \\
             letter &   20000 &  17 &  26 &   \textbf{97.26} &            95.69 &               93.22 &            95.77 &            93.45 &            95.26 &              92.68 \\
              magic &   19020 &  11 &   2 &            86.92 &            89.24 &      \textbf{89.50} &            89.35 &   \textbf{89.42} &            89.19 &              88.33 \\
          miniboone &  130064 &  51 &   2 &            93.07 &            93.21 &      \textbf{93.60} &            93.49 &   \textbf{93.59} &            93.36 &              92.04 \\
           mushroom &    8124 &  22 &   2 &  \textbf{100.00} &  \textbf{100.00} &     \textbf{100.00} &  \textbf{100.00} &  \textbf{100.00} &  \textbf{100.00} &    \textbf{100.00} \\
            nursery &   12960 &   9 &   5 &   \textbf{99.78} &            97.30 &               97.30 &            97.30 &            97.30 &            97.30 &              97.29 \\
        page-blocks &    5473 &  11 &   5 &            95.83 &   \textbf{97.99} &               97.79 &            97.21 &            97.81 &            97.49 &              96.98 \\
          pendigits &   10992 &  17 &  10 &            97.06 &   \textbf{99.65} &      \textbf{99.64} &   \textbf{99.66} &   \textbf{99.64} &   \textbf{99.66} &     \textbf{99.62} \\
           ringnorm &    7400 &  21 &   2 &            97.51 &   \textbf{98.92} &               98.78 &            98.78 &            98.78 &   \textbf{98.86} &     \textbf{98.92} \\
    statlog-landsat &    6435 &  37 &   6 &            91.00 &            90.26 &      \textbf{91.45} &            91.28 &            91.08 &   \textbf{91.35} &              90.35 \\
    statlog-shuttle &   58000 &  10 &   7 &   \textbf{99.90} &   \textbf{99.87} &               99.74 &   \textbf{99.90} &   \textbf{99.81} &            99.79 &              99.80 \\
            thyroid &    7200 &  22 &   3 &            98.16 &            99.41 &      \textbf{99.56} &   \textbf{99.47} &            99.31 &   \textbf{99.52} &              99.13 \\
            twonorm &    7400 &  21 &   2 &   \textbf{98.05} &            97.67 &               97.65 &            97.65 &            97.72 &            97.69 &              97.72 \\
     wall-following &    5456 &  25 &   4 &            90.98 &            94.79 &               95.64 &            95.56 &   \textbf{95.76} &            93.07 &              91.48 \\
           waveform &    5000 &  22 &   3 &            84.80 &            85.80 &               86.54 &            86.13 &            86.21 &            86.53 &     \textbf{87.55} \\
     waveform-noise &    5000 &  41 &   3 &   \textbf{86.08} &            82.59 &               82.71 &            82.93 &            83.12 &            83.05 &              82.71 \\
 wine-quality-white &    4898 &  12 &   7 &   \textbf{63.73} &            57.14 &               58.61 &            57.05 &            59.56 &            56.58 &              55.71 \\
\midrule              Mean &         &     &     &             \textbf{91.6} &             90.4 &                90.2 &             90.6 &             90.3 &             89.9 &               89.0 \\
             Median &         &     &     &             93.1 &             94.8 &                93.6 &             \textbf{95.6} &             93.6 &             93.4 &               92.0 \\
           Avg Rank &         &     &     &             4.16 &             3.89 &                3.53 &             3.68 &             \textbf{3.42} &             3.89 &               5.42 \\
\bottomrule
\end{tabular}
}
\caption{Classification accuracy results, including the results from \citet{klambauer2017self}.} 
\label{table:classification_acc_full}
\end{table}

\begin{table}
\resizebox{\columnwidth}{!}{%
\begin{tabular}{llllllllll}
\toprule
                    &       N &   D &   K &  \textsc{Coupled} & \textsc{CoupledNat} &     \textsc{Orth} &  \textsc{OrthNat} &   \textsc{Hybrid} & \textsc{Decoupled} \\
\midrule
              adult &   48842 &  15 &   2 &           -0.3048 &    \textbf{-0.2970} &           -0.3045 &  \textbf{-0.2973} &           -0.3067 &            -0.3234 \\
         chess-krvk &   28056 &   7 &  18 &  \textbf{-2.1239} &             -3.2821 &           -2.1625 &           -3.2145 &           -3.0443 &            -3.5380 \\
          connect-4 &   67557 &  43 &   2 &           -0.3160 &    \textbf{-0.3009} &           -0.3086 &  \textbf{-0.3017} &           -0.3244 &            -0.3680 \\
             letter &   20000 &  17 &  26 &           -0.2316 &             -0.4892 &  \textbf{-0.2276} &           -0.4793 &           -0.2810 &            -0.4861 \\
              magic &   19020 &  11 &   2 &           -0.2666 &    \textbf{-0.2641} &           -0.2658 &  \textbf{-0.2646} &           -0.2697 &            -0.2863 \\
          miniboone &  130064 &  51 &   2 &           -0.1680 &    \textbf{-0.1584} &           -0.1618 &  \textbf{-0.1585} &           -0.1645 &            -0.1902 \\
           mushroom &    8124 &  22 &   2 &  \textbf{-0.0006} &    \textbf{-0.0007} &  \textbf{-0.0009} &  \textbf{-0.0008} &  \textbf{-0.0007} &   \textbf{-0.0007} \\
            nursery &   12960 &   9 &   5 &  \textbf{-0.2228} &    \textbf{-0.2233} &  \textbf{-0.2225} &  \textbf{-0.2229} &           -0.2236 &            -0.2241 \\
        page-blocks &    5473 &  11 &   5 &           -0.1301 &    \textbf{-0.0989} &           -0.1328 &           -0.1112 &           -0.1368 &            -0.1538 \\
          pendigits &   10992 &  17 &  10 &           -0.0251 &    \textbf{-0.0216} &  \textbf{-0.0209} &  \textbf{-0.0207} &  \textbf{-0.0216} &            -0.0241 \\
           ringnorm &    7400 &  21 &   2 &  \textbf{-0.0345} &             -0.0458 &           -0.0466 &           -0.0465 &           -0.0410 &            -0.0418 \\
    statlog-landsat &    6435 &  37 &   6 &           -0.4503 &             -0.4102 &           -0.3956 &           -0.3920 &  \textbf{-0.3771} &            -0.4938 \\
    statlog-shuttle &   58000 &  10 &   7 &  \textbf{-0.0047} &             -0.0199 &  \textbf{-0.0049} &           -0.0174 &           -0.0170 &            -0.0166 \\
            thyroid &    7200 &  22 &   3 &           -0.0257 &             -0.0133 &  \textbf{-0.0115} &           -0.0211 &           -0.0127 &            -0.0290 \\
            twonorm &    7400 &  21 &   2 &  \textbf{-0.0590} &    \textbf{-0.0588} &  \textbf{-0.0590} &  \textbf{-0.0595} &           -0.0607 &            -0.0620 \\
     wall-following &    5456 &  25 &   4 &           -0.2032 &             -0.1674 &  \textbf{-0.1514} &           -0.1537 &           -0.3191 &            -0.4102 \\
           waveform &    5000 &  22 &   3 &           -0.6207 &             -0.5572 &           -0.5640 &  \textbf{-0.5038} &           -0.5160 &            -0.5469 \\
     waveform-noise &    5000 &  41 &   3 &           -0.7650 &             -0.7416 &           -0.7096 &  \textbf{-0.6778} &           -0.6893 &            -0.7673 \\
 wine-quality-white &    4898 &  12 &   7 &           -2.8884 &             -2.5400 &           -2.5681 &  \textbf{-2.5069} &           -2.7557 &            -2.7921 \\
\midrule               Mean &         &     &     &           -0.4653 &             -0.5100 &           \textbf{-0.4378} &           -0.4974 &           -0.5033 &            -0.5660 \\
             Median &         &     &     &            -0.223 &              -0.223 &            \textbf{-0.222} &            -0.223 &            -0.270 &             -0.286 \\
           Avg Rank &         &     &     &             3.553 &               3.026 &             2.868 &             \textbf{2.737} &             3.605 &              5.211 \\
\bottomrule
\end{tabular}
}
\caption{As Table \ref{table:classification_acc_full} but reporting test log-likelihoods. The test log-likelihood results from \citep{klambauer2017self} were not reported}
\label{table:regression}
\end{table}

\end{document}